\documentclass[onefignum,onetabnum]{siamart171218}

\usepackage{lipsum}
\usepackage{amsfonts}
\usepackage{graphicx}
\usepackage{epstopdf}
\usepackage{multirow}
\usepackage[section]{placeins}
\usepackage{algorithmic}
\usepackage{tikz}
\usepackage{caption}
\usepackage{enumerate}
\usetikzlibrary{shapes,arrows,positioning,shadows}
\usetikzlibrary{decorations}
\usetikzlibrary{decorations.text}
\usetikzlibrary{arrows,decorations.markings}
\usetikzlibrary{decorations.pathreplacing}
\usetikzlibrary{calc}
\tikzset{
    ncbar angle/.initial=90,
    ncbar/.style={
        to path=(\tikztostart)
        -- ($(\tikztostart)!#1!\pgfkeysvalueof{/tikz/ncbar angle}:(\tikztotarget)$)
        -- ($(\tikztotarget)!($(\tikztostart)!#1!\pgfkeysvalueof{/tikz/ncbar angle}:(\tikztotarget)$)!\pgfkeysvalueof{/tikz/ncbar angle}:(\tikztostart)$)
        -- (\tikztotarget)
    },
    ncbar/.default=0.5cm,
}

\tikzset{square left brace/.style={ncbar=0.1cm}}
\tikzset{square right brace/.style={ncbar=-0.1cm}}
\ifpdf
  \DeclareGraphicsExtensions{.eps,.pdf,.png,.jpg}
\else
  \DeclareGraphicsExtensions{.eps}
\fi

\newsiamremark{remark}{Remark}
\newsiamremark{assumption}{Assumption}
\newsiamremark{hypothesis}{Hypothesis}
\crefname{hypothesis}{Hypothesis}{Hypotheses}
\newsiamthm{claim}{Claim}

\newcounter{methodcounter}
\newenvironment{method}[1]
{
\refstepcounter{methodcounter}
  \par\vspace{\baselineskip}\noindent 
  \textbf{Type #1 Methods:}\begin{itshape}%
}%
{
  \end{itshape}\ignorespacesafterend 
}

\title{NINNs: Nudging Induced Neural Networks 
\author{Harbir Antil\thanks{Department of Mathematical Sciences and The Center for Mathematics and Artificial Intelligence, George Mason University, Fairfax, VA 22030 
  (\email{hantil@gmu.edu}).}
\and Rainald L\"ohner\thanks{The Center for Computational Fluid Dynamics, George Mason University, Fairfax, VA 22030 
  (\email{rlohner@gmu.edu}).}
\and Randy Price\thanks{The Center for Mathematics and Artificial Intelligence and The Center for Computational Fluid Dynamics, George Mason University, Fairfax, VA 22030 
  (\email{rprice25@gmu.edu}).}
}
\thanks{This work is partially supported by the Defense Threat Reduction Agency (DTRA) under contract HDTRA1-15-1-0068. Jacqueline Bell served as the technical monitor. Also, it is partially supported by 
NSF grants DMS-2110263, DMS-1913004, the Air Force Office of Scientific Research under Award NO: FA9550-19-1-0036.}}

\usepackage{amsopn}

\makeatletter
\newcommand*{\addFileDependency}[1]{
  \typeout{(#1)}
  \@addtofilelist{#1}
  \IfFileExists{#1}{}{\typeout{No file #1.}}
}
\makeatother

\ifpdf
\hypersetup{
  pdftitle={An Example Article},
  pdfauthor={D. Doe, P. T. Frank, and J. E. Smith}
}
\fi

 \usepackage[textsize=footnotesize]{todonotes}
\definecolor{violet}{rgb}{0.580,0,0.827}

\newcommand*{\bigchi}{\mbox{\Large$\chi$}}

\begin{document}

\maketitle

\begin{abstract}
  New algorithms called nudging induced neural networks (NINNs), to control and improve the accuracy of deep neural networks (DNNs), are introduced. The NINNs framework can be applied to almost all pre-existing DNNs, with forward propagation, with costs comparable to existing DNNs.
  NINNs work by adding a 
  feedback control term to the forward propagation of the network. The feedback term nudges the neural network towards a desired quantity of interest. 
  NINNs offer multiple advantages, for instance, they lead to higher accuracy when compared with existing data assimilation algorithms such as nudging.
  Rigorous convergence analysis is established for NINNs.
  The algorithmic and theoretical findings are illustrated on examples from data assimilation and chemically reacting flows.
\end{abstract}



\section{Introduction}
To illustrate the proposed ideas, let us recall that the Residual neural networks (ResNets) are an established way to do supervised machine learning and their connection to ODEs has helped prove their stability \cite{EHaber_LRuthotto_2018a,KHe_XZhang_SRen_JSun_2016a,LRuthotto_EHaber_2019a}. Multiple authors have made the connection between a general ResNet with input $y_0$, 
\begin{align}\label{generalresnet}
    y_{\ell+1}=y_{\ell} + \tau F(\theta_\ell,y_{\ell})\quad\mbox{for} \quad \ell=0,...,L-1,
\end{align}
where $\theta_\ell$ are the weights and biases for the $\ell$-th layer, and $\tau$ is a positive parameter. The corresponding continuous dynamical system is given by
\begin{align}\label{generalsystem}
    d_t y &= F(\theta(t),y), \quad \quad
    y(0) = y_0. 
\end{align}
The ResNet \eqref{generalresnet} can be seen as the forward Euler discretization of the initial value problem \eqref{generalsystem} \cite{MLODE,LRuthotto_EHaber_2019a,HAntil_HCElman_AOnwunta_DVerma_2021a,antil2020fractional}. In this paper, we leverage the connection between ResNets and ODEs to nudge a ResNet towards a given quantity of interest (QoI), by introducing a feedback law. This approach is hereby termed as Nudging Induced Neural Networks (NINNs). NINNs are applied to data assimilation and realistic chemically reacting flow problems. Also a rigorous convergence analysis is established for NINNs.

Data assimilation techniques are used to improve our knowledge about the state by combining the model with the given observations. The standard nudging algorithm is widely used in data assimilation. In the past nudging was applied to finite-dimensional dynamical systems governed by ordinary differential equations and meteorology \cite{DataAssimilationandInitializationofHurricanePredictionModels,npg-15-305-2008,TheInitializationofNumericalModelsbyaDynamicInitializationTechnique,364173,article,UseofFourDimensionalDataAssimilationinaLimitedAreaMesoscaleModelPartIExperimentswithSynopticScaleData}. As nudging has matured, it has been extended to more general situations, including partial differential equations \cite{Albanez2016ContinuousDA,Bessaih2015ContinuousDA,biswas,Farhat2015ContinuousDA,Farhat2015DataAA,Farhat2016OnTC,Farhat2017ADA,foias2016discrete,Markowich2015ContinuousDA}. Given a continuous dynamical system
\begin{equation}\label{ODE}
    \partial_t u = f(u(t)),
\end{equation}
with unknown initial conditions, the nudging algorithm entails solving
\begin{align}\label{aotsystem}
    \partial_t w &= f(w)-\mu (I_M w-w^{QoI}),\quad\quad
    w(0) = w_0\ \mbox{(arbitrary)},
\end{align}
where $I_M$ is a linear operator called the interpolant operator and $w^{QoI}$ is a quantity of interest (QoI). $I_M$ ensures that $I_Mw$  matches the dimensions of $w^{QoI}$. Here $\mu>0$ is the nudging parameter. In the case $w^{QoI}=I_M(u)$ it is possible to establish approximation error estimates between $u$ and $w$ solving \eqref{ODE} and \eqref{aotsystem}, respectively. Recently in \cite{antil2021data}, the authors replace the discrete version of \eqref{aotsystem} by a ResNet \eqref{generalresnet}. This provides a new and cheaper alternative to nudging as no expensive simulations are needed to generate the nudging solution.  Error estimates have also been derived.

Motivated by nudging \eqref{aotsystem}, this work presents a completely new class of algorithms called NINNs which are meant to directly control DNNs, such as ResNets \eqref{generalresnet}, by appropriately applying nudging (feedback). 
One example of NINNs is
\begin{equation}
    y_{\ell+1} := y_\ell + \tau F(\theta_\ell,y_{\ell})- \tau\mu\bigchi_{\Omega_{QoI}}(y_\ell-  y_{\ell+1}^{RQoI}),\quad \ell=1,...,L-2 . 
\end{equation}
 Notice that, the NINNs is applied to the forward propagation \eqref{generalresnet} after the training has been carried out and is therefore applicable to almost any network, in particular, ResNets. Furthermore, the NINNs framework is a type of feedback control for DNNs. This framework offers multiple advantages, for example, it leads to new data assimilation algorithms with similar or better performance than the nudging data assimilation algorithm in terms of accuracy and run time.  

Performance of NINNs is illustrated using two types of examples. In the first example, we use NINNs as a data assimilation algorithm. In the second example, we consider a realistic chemically reacting flow problem to illustrate uses for NINNs in simulating stiff ODEs. 
In case of stiff or chaotic systems \cite{brown2021novel}, a proven technique is to learn the update map $u^{n-1}\to u^n$, i.e., one time step solution. The neural net takes as input the state at time $t_{n-1}$ and the output is the approximate state at time $t_n$. This is the approach we will take for the ResNets in the numerical section, the details are given in section \ref{sec:biasordering}.

\smallskip
\noindent
{\bf Outline:} Section \ref{sec:2} contains preliminary material which is necessary to introduce NINNs in section \ref{sec:DNNnudging}. Section \ref{sec:analysis} contains error analysis of NINNs by means of two theorems. 
Next we provide experimental results from using NINNs as a data assimilation algorithm in section \ref{sec:DA}. In section \ref{sec:chem}, we provide experimental results to learn stiff ODEs arising in chemically reacting flows.

\section{Objectives and Definitions}\label{sec:2}
The goal of this section is to explain the nudging algorithm in more detail and to describe the ResNet architecture so that in section \ref{sec:DNNnudging} we can smoothly transition into NINNs. The contents of the next three subsections \ref{sec:Generalnudging}-\ref{sec:DNNstruct} are by now  well-known, see for instance \cite{antil2021data,antil2021deep,brown2021novel,antil2020fractional}.

\subsection{Nudging}\label{sec:Generalnudging}
Consider a dynamical system given by a differential equation
\begin{equation}\label{ODE2}
    \partial_t u = f(u(t)),
\end{equation}
with $u(t)\in X$. 
Here $X$ is either an infinite dimensional space (in case of PDEs) or a finite dimensional space $X = \mathbb{R}^d$ in case of vector ODEs. 
The continuous nudging algorithm with a continuous in time quantity of interest (QoI) $w^{QoI}$ is given by:
\begin{equation} \label{nudging_cont}
\partial_t w = f(w)-\mu (I_M w- w^{QoI}),\quad w(0)=w_0\ (\mbox{arbitrary}) . 
\end{equation}
More realistic is the case with discrete QoI at times $\{t_n\}$:
\begin{align}\label{nudging}
\begin{aligned}
    \partial_t w &= f(w(t)) - \mu(I_M(w(t))-w^{QoI}(t_{n})),\quad \forall t\in[t_n,t_{n+1}], \quad  
    w(0)=w_0.
\end{aligned}    
\end{align}
Notice that, in both cases we have QoI of the form $\{w^{QoI}(t)\}$ or $\{ w^{QoI}(t_n)\}$ which are an input to the algorithm. Moreover, $I_M$ is an interpolant operator acting on $X$. In the case of ODEs, an example of $I_M$ is the orthogonal projection operator onto a subset $K$ of $X$. In this case $X=\mathbb{R}^d$ and $K=\mbox{span} \{ \phi_i \}_{i=1}^{d_K}$ with  $\phi_i = \textbf{e}_i$ the standard basis elements in $\mathbb{R}^d$ and $d\ge d_K$.

The goal of nudging is to nudge the solution of \eqref{nudging_cont} and \eqref{nudging} toward the quantity of interest. A typical QoI is a particular solution of \eqref{ODE2}, i.e., $w^{QoI}=I_M(u)$, namely we are nudging $w$ towards a particular solution. This is useful for data assimilation where the initial condition is an unknown and the user is given partial observations of $u$.

 \subsection{DNN with bias ordering}\label{sec:biasordering} 
 
The next two subsections cover the ResNet details. We are particularly interested in ResNets that approximate a dynamical system \eqref{ODE}. A proven technique is to teach the ResNet the update map $u(t_n)\to u(t_{n+1})$ which we denote by $S$ below. First introduced in \cite{antil2021deep}, consider the following optimization problem which abstractly represents the training of the DNNs
\begin{subequations}
\begin{align}\label{NN}
    \min_{\{W_\ell\}_{\ell=0}^{L-1},\{b_\ell\}_{\ell=0}^{L-2}} J(\{(y_L^i,S(u^i))\}_i,\{W_\ell\}_\ell,\{b_\ell\}_\ell), \\
    \text{subject to }y_L^i = \mathcal{F}(u^i;(\{W_\ell\},\{b_\ell\})),\quad i=1,...,N_s, \label{DNN} \\
    b_\ell^j\le b_\ell^{j+1},\quad j=1,...,n_{\ell+1}-1,\quad \ell=0,...,L-2. \label{biasordering} 
\end{align}
\end{subequations}
The input-output pairs used in training are represented by $\{u^i\}_{i=1}^{N_s}$ and $\{S(u^i)\}_{i=1}^{N_s}$, with $N_s$ denoting the number of samples. The goal is to minimize the difference between the DNN output (DNN is represented by $\mathcal{F}$ in \eqref{DNN}) $y^i_L$ and the true output $S(u^i)$ using the loss function $J$ from \eqref{NN}. Bias ordering is enforced in each layer by \eqref{biasordering}, see \cite{antil2021deep} for more details. The weight matrix is $W_\ell\in\mathbb{R}^{n_\ell\times n_{\ell+1}}$, and the bias vector is $b_\ell\in\mathbb{R}^{n_{\ell +1}}$ where the $\ell$-th layer has $n_\ell$ neurons. 

In our numerical experiments the loss function $J$ in \eqref{NN} will be quadratic
\begin{align}\label{lossf}
J:=\frac{1}{2N}\sum_{i=1}^N\|y_L^i-S(u^i)\|_2^2 + \frac{\lambda}{2}\sum_{\ell=0}^{L-1}(\|W_\ell\|_1 + \|b_\ell\|_1 + \|W_\ell\|_2^2 + \|b_\ell\|_2^2) ,
\end{align}
where $\lambda\ge0$ is the regularization parameter. The second summation regularizes the weights and biases. Following \cite{antil2021deep}, and motivated by Moreau-Yosida regularization, the bias ordering \eqref{biasordering} is implemented as an additional penalty term in $J$
\begin{equation}
    J_{\gamma} := J + \frac{\gamma}{2}\sum_{\ell=0}^{L-2}\sum_{j=1}^{n_{\ell+1}-1}\|\min\{b_{\ell}^{j+1}-b_{\ell}^j,0\}\|_2^2 .
\end{equation}
Here $\gamma$ is a penalization parameter. For convergence results as $\gamma \rightarrow \infty$, see \cite{antil2021deep}.

\subsection{DNN Structure}\label{sec:DNNstruct}

To introduce NINNs, we will use ResNets with the following form with input $y_0\in\mathbb{R}^d$, inner layer feature vectors $y_\ell\in\mathbb{R}^{n_\ell}$, and output feature vector $y_L\in\mathbb{R}^{d^*}$ 
\begin{align}\label{ResNet}
\begin{aligned}
    y_1 &:= \sigma(W_0y_0+b_0),\\ 
    y_{\ell+1} &:= y_\ell + \tau\sigma(W_\ell y_\ell + b_\ell),\quad \ell=1,...,L-2,\\ 
    y_L &:= W_{L-1}y_{L-1} .
\end{aligned}    
\end{align}
The scalar $\tau>0$ and the activation function $\sigma$ are user defined. For the purpose of this work, we have chosen a smooth quadratic approximation of the ReLU function,
\begin{equation*}\label{sigma}
 \sigma(x) =
    \begin{cases}
     \max\{0,x\} & |x|>\epsilon,\\
      \frac{1}{4\epsilon}x^2+\frac{1}{2}x+\frac{\epsilon}{4} & |x|\le\epsilon.\\
    \end{cases}     
\end{equation*}

The objective of this work is to show that the ResNet \eqref{ResNet} output can be effectively nudged towards a given QoI using a setup similar to standard nudging described in section \ref{sec:Generalnudging}. 

\section{NINNs}\label{sec:DNNnudging} The user provides a trained ResNet of the form $\eqref{ResNet}$ and a QoI $y^{QoI}$. Then a general NINN introduces a feedback law into the ResNet:
\begin{align}\label{NINNs}
\begin{aligned}
    y_1 &:= \sigma(W_0y_0+b_0)-\tau\mu g_1(y^{QoI},y_1),\\ 
    y_{\ell+1} &:= y_\ell + \tau\sigma(W_\ell y_\ell + b_\ell)-\tau\mu g_{\ell+1}(y^{QoI},y_\ell),\quad \ell=1,...,L-2,\\ 
    y_L &:= W_{L-1}y_{L-1} .
\end{aligned}    
\end{align}
The goal is to choose functions $\{g_\ell\}_{\ell=1}^{L-1}$ and parameter $\mu\in\mathbb{R}$ such that the output $y_L$ of \eqref{NINNs} is nudged towards $y^{QoI}$. Next we present choices for the functions $g_\ell$ and provide the details in the following subsections.

\begin{method}{1}\label{Type1methods}
Choose $g_1:=y_1-y_{1}^{RQoI}$ and $g_{\ell+1}:=y_\ell-y_{\ell+1}^{RQoI}$ for $\ell=1,\dots,L-2$. $\{y_\ell^{RQoI}\}_{\ell=1}^{L-1}$ is generated from the given $y^{QoI}$ by the user.
\end{method}

\begin{method}{2}\label{Type2methods}
Choose $g_1:= 0$ and $g_{\ell+1}:=N_\ell(y_\ell,y^{QoI})z_\ell$ for $\ell=1,\dots,L-2$. Function $N_\ell:\mathbb{R}^{n_\ell}\mapsto\mathbb{R}$ and vector $z_\ell$ are chosen by the user.
\end{method}


\subsection{Type 1 Methods}\label{type1methods}

Type 1 methods require access to $\{y_\ell^{RQoI}\}_{\ell=1}^{L-1}$. We compute it by passing the user specified $y^{QoI}$ as an input into the trained ResNet \eqref{ResNet}. Then, we define  
\[
    y^{RQoI}_{\ell} := y_\ell , \quad \ell = 1, \dots, L-1 .
\]
However, care must be observed since $y^{QoI}$ in general will not be $d$-dimensional. Recall that the input to ResNet \eqref{ResNet} lies in $\mathbb{R}^d$. In this case, the user can combine a valid input from a previous iteration with $y^{QoI}$ to generate $ y^{RQoI}_{\ell}$. In Figure \ref{fig:ResNetMethod1} a ResNet with input/output in $\mathbb{R}^3$ is pictured. The user desires the second component of the output to be nudged towards $y^{QoI}$. The right panel depicts how the user can generate the $y^{RQoI}_{\ell}$. More details are provided in the numerical section.


\pgfdeclarelayer{background}
\pgfdeclarelayer{foreground}
\pgfsetlayers{background,main,foreground}
\tikzstyle{mybox} = [draw=black, fill=blue!30, very thick,
    rectangle, rounded corners, inner sep=1pt, inner ysep=15pt]
  \begin{figure}[htb]
  \centering
{\footnotesize
\begin{tikzpicture}[
dot/.style = {circle, fill, minimum size=#1,
              inner sep=0pt, outer sep=0pt},
dot/.default = 6pt  
]
\draw [black, thick] (0,0) to [square left brace ] (0,2);
\draw [black, thick] (.5,0) to [square right brace] (.5,2);

\draw [black, thick] (4.1,0) to [square left brace ] (4.1,2);
\draw [black, thick] (4.6,0) to [square right brace] (4.6,2);

\draw [black, thick] (7,0) to [square left brace ] (7,2);
\draw [black, thick] (7.5,0) to [square right brace] (7.5,2);

\node [mybox] (box2) at (2.3,1) {%
    \begin{minipage}{50pt}\centering
        ResNet 
    \end{minipage}
};
\node [mybox] (box2) at (9.3,1) {%
    \begin{minipage}{50pt}\centering
        ResNet 
    \end{minipage}
};

\node (ResNetSystem) at (.25,-.35) {\footnotesize Input};
\node (ResNetSystem) at (4.35,-.35) {\footnotesize Output};
\node (ResNetSystem) at (7.25,-.35) {\footnotesize Valid Input};
\node [text=red](ResNetSystem) at (5.2,1) {\large $y^{QoI}$};
\node [text=red](ResNetSystem) at (7.25,1.2) {$y^{QoI}$};
\node [text=red](ResNetSystem) at (9.3,-.5) {\large $\{y_{\ell}^{RQoI}\}_{\ell=1}^{L-1}$};

\draw[dotted,->, line width=1.5pt] (.75,1) -- (1.25,1);
\draw[dotted,->, line width=1.5pt] (3.35,1) -- (3.85,1);
\draw[dotted,->, line width=1.5pt] (7.75,1) -- (8.25,1);
\draw[dotted,->, line width=1.5pt] (9.3,.25) -- (9.3,-.25);
\draw[->, line width=1pt] (4.2,1) -- (4.5,1);

\node[dot=3pt] at (.25,1.75) {};
\node[dot=3pt] at (.25,1) {};
\node[dot=3pt] at (.25,.25) {};

\node[dot=3pt] at (4.35,1.75) {};
\node[dot=3pt] at (4.1,1) {};
\node[dot=3pt] at (4.35,.25) {};

\node[dot=3pt] at (7.25,1.75) {};
\node[dot=3pt] at (7.25,.25) {};
\node[dot=3pt, fill=red] at (7.25,1) {};
\node[dot=4pt, fill=red] at (4.6,1) {};

\end{tikzpicture}
}
 \caption{{(\bf Type 1 Methods)} The left panel depicts a ResNet with input/output in $\mathbb{R}^3$. The goal is to nudge the ResNet output towards the given $y^{QoI}$ (second component).
 The right panel show generation of $\{y^{RQoI}\}$ from $y^{QoI}$ by passing it through a trained ResNet}.
\label{fig:ResNetMethod1}
\end{figure}
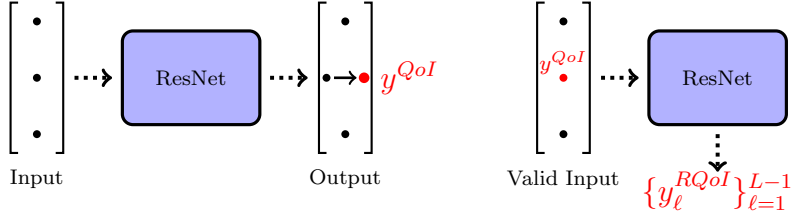

\subsection{Type 2 Methods} \label{type2methods}

Type 2 methods require defining $N_\ell$ and $z_\ell$. Let $f_{\ell+1}(x):= y_\ell + \tau\sigma(W_\ell x + b_\ell)$ for $\ell=1,\dots,L-2$ and $f_L(x):=W_{L-1}x$. Next, we consider multiple cases: (a) $y_L\in\mathbb{R}$ and (b) $y_L\in\mathbb{R}^n$.

\noindent
{\bf Case 1:} ($y_L\in\mathbb{R}$): Set $z_\ell = \frac{1}{\|W_{L-1}\|_1 }W_{L-1}$. Then the user has the option of choosing $N_\ell=W_{L-1}y_\ell-y^{QoI}$ or
$N_\ell=W_{L-1}f_{L}\circ\dots\circ f_{\ell+1}(y_\ell) - y^{QoI}$. Therefore $z_\ell\in\mathbb{R}^{n_\ell}$, $N_\ell\in\mathbb{R}$ and $y^{QoI}\in\mathbb{R}$.

\medskip

\noindent
{\bf Case 2:} 
($y_L\in\mathbb{R}^n$): Choose $z_{\ell} := \mbox{argmin}_{|x|\le 1} |W_{L-1}(y_\ell+x) - y^{QoI}|^2$ or $z_{\ell} := \mbox{argmin}_{|x|\le 1} |W_{L-1}(f_{L}\circ\dots\circ f_{\ell+1}(y_\ell)+x) - y^{QoI}|^2$. Then the user has the option of choosing $N_\ell=|(W_{L-1}y_\ell-y^{QoI})|$ or $N_\ell=|W_{L-1}f_{L}\circ\dots\circ f_{\ell+1}(y_\ell) - y^{QoI}|$. Throughout the paper, $|\cdot|$ represents the 2-norm. Therefore $z_\ell\in\mathbb{R}^{n_\ell}$, $N_\ell\in\mathbb{R}$ and $y^{QoI}\in\mathbb{R}^{n}$.

\subsection{System of ResNets}
To apply the NINNs framework \eqref{NINNs} to a system of ResNets, such as shown in Figure \ref{fig:ResNetsystem}, we simply require applying the framework to each ResNet in the system the user desires to control. To control the first component of the output pictured in Figure \ref{fig:ResNetsystem} requires applying the NINNs framework to ResNet \#1 only.

\pgfdeclarelayer{background}
\pgfdeclarelayer{foreground}
\pgfsetlayers{background,main,foreground}
\tikzstyle{mybox} = [draw=black, fill=blue!30, very thick,
    rectangle, rounded corners, inner sep=5pt, inner ysep=15pt]
  \begin{figure}[htb]
  \centering
{\footnotesize
\begin{tikzpicture}[
dot/.style = {circle, fill, minimum size=#1,
              inner sep=0pt, outer sep=0pt},
dot/.default = 6pt  
]
\draw [black, thick] (0,0) to [square left brace ] (0,2);
\draw [black, thick] (.5,0) to [square right brace] (.5,2);

\draw [black, thick] (5.5,0) to [square left brace ] (5.5,2);
\draw [black, thick] (6,0) to [square right brace] (6,2);

\node [mybox] (box1) at (3,2.75) {%
    \begin{minipage}{50pt}
        ResNet \# 1
    \end{minipage}
};
\node [mybox] (box2) at (3,1) {%
    \begin{minipage}{50pt}
        ResNet \# 2
    \end{minipage}
};
\node [mybox] (box3) at (3,-.75) {%
    \begin{minipage}{50pt}
        ResNet \# 3
    \end{minipage}
};
\node (ResNetSystem) at (3,-2) {\footnotesize ResNet System};
\node (ResNetSystem) at (.25,-.35) {\footnotesize Input};
\node (ResNetSystem) at (5.75,-.35) {\footnotesize Output};
\draw[dotted,->, line width=1.5pt] (.75,1) -- (1.9,1);
\draw[dotted,->, line width=1.5pt] (.75,1) -- (1.85,2.75);
\draw[dotted,->, line width=1.5pt] (.75,1) -- (1.85,-.75);

\draw[dotted,->, line width=1.5pt] (4.05,1) -- (5.65,1);
\draw[dotted,->, line width=1.5pt] (4.05,2.75) -- (5.65,1.79);
\draw[dotted,->, line width=1.5pt] (4.05,-.75) -- (5.65,.21);

\node[dot=3pt] at (.25,1.75) {};
\node[dot=3pt] at (.25,1) {};
\node[dot=3pt] at (.25,.25) {};

\node[dot=3pt] at (5.75,1.75) {};
\node[dot=3pt] at (5.75,1) {};
\node[dot=3pt] at (5.75,.25) {};

 \begin{pgfonlayer}{background}
        \path (1.75,3.65) node (a) {};
        \path (4.25,-2.5) node (b) {};
        \path[fill=blue!10,rounded corners, draw=black!50, dashed]
            (a) rectangle (b);
    \end{pgfonlayer}
\end{tikzpicture}
}
 \caption{A ResNet system consisting of 3 ResNets. Each ResNet is responsible for one component of the output. Nudging a component of the output towards a QoI requires the user to apply NINNs to the corresponding ResNet.}
\label{fig:ResNetsystem}
\end{figure}
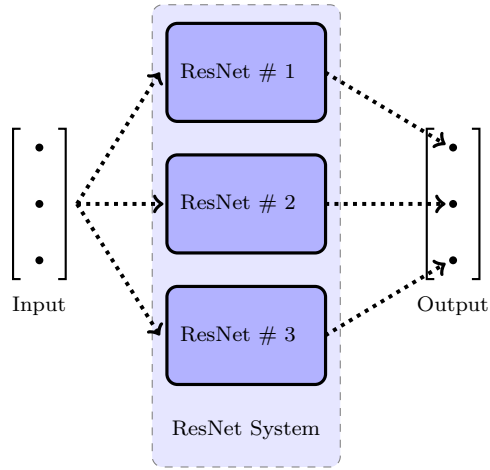
\section{Error Analysis}
In this section we provide error analysis resulting from applying NINNs to ResNets trained to learn a dynamical system,
\begin{equation}\label{refdynamics}
    \partial_t u = f(u(t)),\quad u(0)=u_0.
\end{equation}
We will assume $u$ is finite dimensional, either by construction or discretization. Given  partial/incomplete QoI (or observations) taken from a dynamical system \eqref{refdynamics} solution, we will show NINNs can recover the solution under certain assumptions. First, we introduce the necessary definitions and assumptions.
\label{sec:analysis}
  \begin{definition}[ResNet state space]\label{def:ResNetstatespace}
  Consider a ResNet with inner layers of size $n_\ell$ with $\ell = 1,\dots, L-2$. The ResNet state space is given by $\mathbb{R}^{n_\ell}$. If the ResNet is a collection of $N_R$ ResNets, then we define $n_\ell=\sum_{i=1}^{N_R} n_{\ell_i}$, where $n_{\ell_i}$ is the size of the $\ell$-th inner layer for the $i$-th ResNet.
  \end{definition} 
  Figure \ref{fig:ResNetsystemcomponents} shows a ResNet system where each ResNet contains two layers with three neurons (also known as components) each. The ResNet state space is $\mathbb{R}^9$. The components of the ResNet state space ordering is shown in the figure. 

\pgfdeclarelayer{background}
\pgfdeclarelayer{foreground}
\pgfsetlayers{background,main,foreground}
\tikzset{My Arrow Style/.style={single arrow, fill=red!50, anchor=base, align=center,text width=1cm}}
\tikzstyle{mybox} = [draw=black, fill=blue!30, very thick,
    rectangle, rounded corners, inner sep=5pt, inner ysep=15pt]
  \begin{figure}[htb]
  \centering
{\footnotesize
\begin{tikzpicture}[ultra thick,
dot/.style = {circle, fill, minimum size=#1,
              inner sep=0pt, outer sep=0pt},
dot/.default = 6pt  
]

\node [mybox] (box1) at (0,2.75) {%
    \begin{minipage}{50pt}
        ResNet \# 1
    \end{minipage}
};
\node [mybox] (box2) at (0,1) {%
    \begin{minipage}{50pt}
        ResNet \# 2
    \end{minipage}
};
\node [mybox] (box3) at (0,-.75) {%
    \begin{minipage}{50pt}
        ResNet \# 3
    \end{minipage}
};
\node [mybox] (box4) at (4,2.75) {%
    \begin{minipage}{50pt}
        \text{      }
    \end{minipage}
};
\node [mybox] (box5) at (4,1) {%
    \begin{minipage}{50pt}
        \text{      }
    \end{minipage}
};
\node [mybox] (box6) at (4,-.75) {%
    \begin{minipage}{50pt}
        \text{      }
    \end{minipage}
};
\node (ResNetSystem) at (0,-2) {\footnotesize ResNet System};

\node[My Arrow Style] at (2,.9) {\text{ }};

\node[dot=3pt] at (3.5,3.1) {};
\node[dot=3pt] at (3.5,2.75) {};
\node[dot=3pt] at (3.5,2.4) {};

\node[dot=3pt] at (3.5,-.4) {};
\node[dot=3pt] at (3.5,-.75) {};
\node[dot=3pt] at (3.5,-1.1) {};

\node[dot=3pt] at (3.5,1.35) {};
\node[dot=3pt] at (3.5,1) {};
\node[dot=3pt] at (3.5,.65) {};

\node[dot=3pt] at (4.5,3.1) {};
\node[dot=3pt] at (4.5,2.75) {};
\node[dot=3pt] at (4.5,2.4) {};

\node[dot=3pt] at (4.5,1.35) {};
\node[dot=3pt] at (4.5,1) {};
\node[dot=3pt] at (4.5,.65) {};

\node[dot=3pt] at (4.5,-.4) {};
\node[dot=3pt] at (4.5,-.75) {};
\node[dot=3pt] at (4.5,-1.1) {};

\node [text=red,text width=2cm](ResNetSystem) at (3.65,-2) {ResNet State Space in $\mathbb{R}^9$};
\node [text width=2cm](ResNetSystem1) at (6.5,1) {Components 4,5,6 of the ResNet state space};

\draw [decorate,
    decoration = {brace,mirror}] (5.2,.5) --  (5.2,1.5);

 \begin{pgfonlayer}{background}
        \path (-1.25,3.65) node (a) {};
        \path (1.25,-2.5) node (b) {};
        \path[fill=blue!10,rounded corners, draw=black!50, dashed]
            (a) rectangle (b);
        
    \end{pgfonlayer}
\begin{pgfonlayer}{foreground}
       \path (3.2,3.55) node (c) {};
        \path (3.8,-1.55) node (d) {};
        \path[rounded corners, draw=red!100, dashed]
            (c) rectangle (d);
\end{pgfonlayer}
\end{tikzpicture}
}
 \caption{A ResNet system consisting of 3 ResNets. Each ResNet contains two layers with three neurons in each layer. The ResNet state space is $\mathbb{R}^9$. Components of the ResNet state space are labeled from top to bottom.}
\label{fig:ResNetsystemcomponents}
\end{figure}
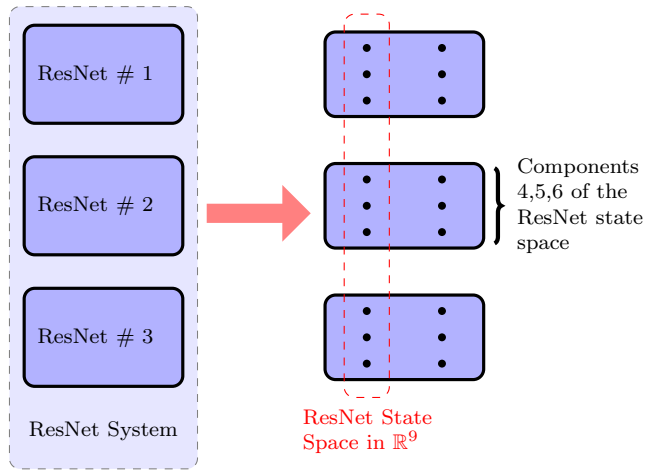
We introduce $\bigchi_{\Omega_{QoI}}:\mathbb{R}^{n_\ell}\mapsto\mathbb{R}^{n_\ell}$, as an indicator type function that sets components without any NINNs feedback law to zero. Let $x\in\mathbb{R}^{n_\ell}$, then the $i$-th component of $\bigchi_{\Omega_{QoI}}(x)$ is defined as:
\begin{equation}\label{bigchi}
\big[\bigchi_{\Omega_{QoI}}(x)\big]_i =
    \begin{cases}
     0 & \mbox{if }i\notin\Omega_{QoI},\\
     [x]_i & \mbox{if }i\in\Omega_{QoI}, 
    \end{cases} 
\end{equation}
where
  \begin{equation}\label{def:Omega}
    \Omega_{QoI} = \{i\in\{1,2,...,n_\ell\} | \mbox{ $i$-th component of ResNet state space is nudged.}\}.
\end{equation}
The notation $[x]_i$ signifies the $i$-th component of a vector. It is easy to see $\bigchi_{\Omega_{QoI}}$ has the following property,
\begin{align}\label{I_Nineq}
    |\bigchi_{\Omega_{QoI}}(x)|\le |x|,\quad\forall x\in\mathbb{R}^{n_\ell}.
\end{align}
$|\cdot|$ represents the 2-norm in this paper.
We re-introduce the interpolant operator $I_M:\mathbb{R}^d\mapsto\mathbb{R}^d$ where $d$ accounts for the dimension of $u$ in \eqref{refdynamics}, recall \eqref{nudging_cont}.
  \begin{assumption}[Stability of $I_{M}$]\label{def:c_M} 
  For all $x$ in $\mathbb{R}^d$, there exists a constant $c_M\ge 0$ such that
  $|I_M(x)-x|\le c_M|x|$.
  \end{assumption}
The partial/incomplete QoI (or observations) taken from \eqref{refdynamics} will be in the form $\{I_M(u(k\Delta t))\}_{k=0}^{\infty}$ where $\Delta t\in\mathbb{R}^+$ is a positive real number. Next we introduce the concept of a \emph{continuous ResNet}, i.e., a ResNet with an infinite amount of layers. The forward propagation through the continuous ResNet is determined by $W(t):\mathbb{R}\mapsto\mathbb{R}^{n_\ell\times n_\ell}$, $b(t):\mathbb{R}\mapsto\mathbb{R}^{n_\ell}$ and activation function $\sigma:\mathbb{R}^{n_\ell}\mapsto\mathbb{R}^{n_\ell}$. We combine these functions into $f_{NN}(x,t)=\sigma(W(t)x + b(t))$. We represent the whole continuous ResNet by the function  $h(t):[0,\infty)\mapsto\mathbb{R}^{n_\ell}$ satisfying
\begin{align}\label{continuousResNet}
\begin{aligned}
    \partial_t h = f_{NN}(h(t),t),\quad h(0)=x_0.\\
    h(t) = \int_0^t f_{NN}(h(s),s)ds + x_0.
\end{aligned}    
\end{align}
If we discretize \eqref{continuousResNet} using Forward Euler then we arrive at \eqref{ResNet}.
\begin{assumption}[Continuous ResNet]\label{contResNet}
There exists a continuous ResNet that, given $u(t)$ at time $t \ge 0$, can replicate the behavior of $u(t+\Delta t)$ (solution to \eqref{refdynamics})  with $\Delta t > 0$.
\begin{enumerate}[(i)]
\item The ResNet input and output lies in $\mathbb{R}^{d}$. \footnote{In this section we are considering ResNets that replicate dynamical systems. In this setting is it natural for the input and output dimensions to be equal. Our analysis can be easily extended to handle the case of additional input parameters. }
\item The layer width $n_\ell$ is the same except for the input and output.
\item The forward propagation 
is given by \eqref{continuousResNet} with $f_{NN}$ being Lipschitz continuous with Lipschitz constant $\widetilde{K}$.
\end{enumerate}
\end{assumption}
\begin{definition}[Input-Output Transformation]\label{def:L}
  Define, $L:=L_{in}\circ L_{out}$. $L_{out}:\mathbb{R}^{n_\ell}\mapsto\mathbb{R}^{d}$ is a linear transformation with Lipschitz constant $K_{out}$ that maps objects from the ResNet state space ($\mathbb{R}^{n_\ell}$) into the ResNet output space ($\mathbb{R}^{d}$). $L_{in}:\mathbb{R}^{d}\mapsto\mathbb{R}^{n_\ell}$ is a transformation, with Lipschitz constant $K_{in}$, that maps objects from the ResNet input space ($\mathbb{R}^{d}$) into the ResNet state space ($\mathbb{R}^{n_\ell}$). L has Lipschitz constant $K_L=K_{in}K_{out}$.
  \end{definition}
  In the notation of \eqref{ResNet}, the transformation $L$ is the action of $W_{L-1}$ (=:$L_{out}$) and $\sigma(W_0y_0+b_0)$ (=:$L_{in}$). It is reasonable to assume that $L$ is Lipschitz with Lipschitz constant $K_L = K_{in}K_{out}$. For instance, the action of $W_{L-1}$ is linear and when $\sigma$ is ReLU, $L_{in}$ is Lipschitz. These transformations are shown in Figure~\ref{diag:online} where the evolution of an arbitrary initial condition through the ResNet state space is depicted.
  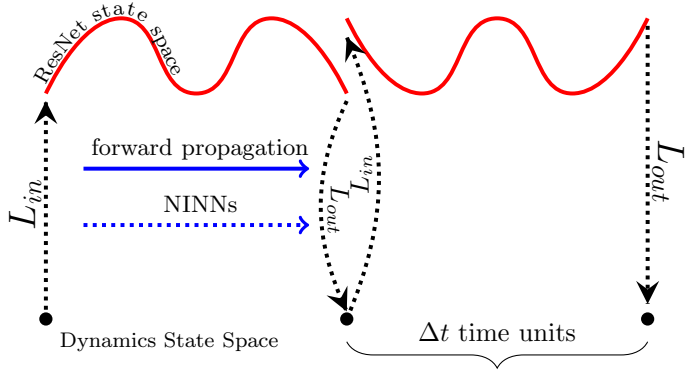
\begin{figure}[htb]
  \centering
  \begin{tikzpicture}[
dot/.style = {circle, fill, minimum size=#1,
              inner sep=0pt, outer sep=0pt},
dot/.default = 6pt  
                    ]
    \tikzset{myptr/.style={decoration={markings,mark=at position 1 with %
    {\arrow[scale=1.5,>=stealth]{>}}},postaction={decorate}}}
    \node[dot=5pt,label=below right:{\footnotesize Dynamics State Space}] at (0,-3) {};
    \node[dot=5pt] at (4,-3) {};
    \node[dot=5pt] at (8,-3) {};
    \draw [red,line width=1.5pt,postaction={decorate,decoration={text along path, raise=1.5pt,
  text={|\footnotesize| ResNet state space},text color=black}}] plot [smooth, tension=1] coordinates {(0,0) (1,1) (2,0) (3,1) (4,0)};
  \draw [red,line width=1.5pt] plot [smooth, tension=1] coordinates {(4,1) (5,0) (6,1) (7,0) (8,1)};

\draw[dotted,myptr, line width=1.5pt, postaction={decorate,decoration={text along path, raise=4pt,
  text={|\Large| {$L_{in}$}},text align={center},text color=black}}] (0,-3) -- (0,-.1);
\draw[dotted,myptr, line width=1.5pt, postaction={decorate,decoration={text along path, raise=4pt,
  text={|\normalsize| {$L_{out}$}},text align={center},text color=black}}] (4,-.1) to[bend right=25] (4,-2.9);
  \draw[dotted,myptr, line width=1.5pt, postaction={decorate,decoration={text along path, raise=4pt,
  text={|\normalsize| {$L_{in}$}},text align={center},text color=black}}] (4,-3) to[bend right=23] (4,.75);
  \draw[dotted,myptr, line width=1.5pt, postaction={decorate,decoration={text along path, raise=4pt,
  text={|\Large| {$L_{out}$}},text align={center},text color=black}}] (8,.9) -- (8,-2.8);
  
  \draw[blue,->, line width=1.5pt,postaction={decorate,decoration={text along path, raise=5pt,
  text={|\small| {forward propagation}},text align={center},text color=black}}] (.5,-1) -- (3.5,-1);
  \draw[dotted, blue,->, line width=1.5pt,postaction={decorate,decoration={text along path, raise=5pt,
  text={|\small| {NINNs}},text align={center},text color=black}}] (.5,-1.75) -- (3.5,-1.75);
  \draw [decorate,decoration={brace,amplitude=10pt,mirror,raise=2ex}]
  (4,-3) -- (8,-3) node[midway,yshift=-.5em]{$\Delta t$ time units};
  \end{tikzpicture}
  \caption{ The figure depicts the evolution of an arbitrary initial condition through the trained ResNet.}
\label{diag:online}
\end{figure}

Next, we rewrite the continuous ResNet \eqref{continuousResNet} with input coming from the \eqref{refdynamics}. This will be crucial to compare NINNs with $u$ solving \eqref{refdynamics}.
\begin{align}\label{contNN}
\begin{aligned}
    \partial_t v &= f_{NN}(v(t)),\qquad\forall t\in[k\Delta t, (k+1)\Delta t), \\ 
    v(k\Delta t) &= L_{in}(u(k\Delta t)), \\ 
    v(0) &= u_0,
\end{aligned}    
\end{align}
 with $v\in \mathbb{R}^{n_\ell}$ and $L_{in}$ given in Definition \ref{def:L}. 
We define the map $F_{NN}:\mathbb{R}^{d}\times[0,\Delta t]\to\mathbb{R}^{n_\ell}$ as follows. $F_{NN}$ takes the input $x_0\in \mathbb{R}^{d}$ and applies the trained continuous ResNet (Assumption \ref{contResNet}) to produce the evolution in time. In other words, 
\begin{equation}\label{eq:FNN}
F_{NN}(x,t) := h(t) = \int_0^t f_{NN}(h(s),s)ds + L_{in}(x_0).
\end{equation}
This map is Lipschitz in the first component with Lipschitz constant $K$, we establish this next.
Before we do that, it will be useful for the rest of the section to introduce the notation $z^{-}$ to indicate a limit from the left and the following ResNet error term.
  \begin{definition}[ResNet error$\  \epsilon_{NN}$]\label{def:ep_NN}
 The ResNet error term $\epsilon_{NN}$ is the maximum value of $|L_{out}(v^{-}(k\Delta t))-u(k\Delta t)|^2$ across all $k\in\mathbb{Z}^{+}$, where $u$ and $v$ are respectively given by \eqref{refdynamics} and \eqref{contNN}. 
\end{definition}

\begin{theorem}\label{thm:LipFNN}
Assume that Assumption \ref{contResNet} holds and recall $L_{in}$ is Lipschitz with Lipschitz constant $K_{in}$ from Definition~\ref{def:L}.
Then the map $(x,t) \mapsto F_{NN}(x,t):=h$ (cf.~\eqref{eq:FNN}) 
 with 
$x \in \mathbb{R}^{d}$ denoting the input is Lipschitz
in the first component with Lipschitz constant 
$K=e^{\widetilde{K}\Delta t}K_{in}$.
\end{theorem}
\begin{proof}
Let $x_1,x_2\in \mathbb{R}^d$. Lifting $x_1,x_2$ to $\mathbb{R}^{n_\ell}$ with $L_{in}$ and denoting the two separate evolution's in \eqref{contNN} by $h_1,h_2$ with $\widetilde{h}=h_1-h_2$, we obtain that
\[
    \partial_t \widetilde{h} = f_{NN}(h_1)-f_{NN}(h_2).
\]
Multiply both sides by $\widetilde{h}$ and using the Lipschitz property of $f_{NN}$,
\[
\frac{1}{2} \partial_t |\widetilde{h}|^2 - \widetilde{K}|\widetilde{h}|^2\le 0.
\]
After integrating,
\[
|\widetilde{h}(t)|^2\le e^{2\widetilde{K}\Delta t} |\widetilde{h}(0)|=e^{2\widetilde{K}\Delta t}|L_{in}(x_1)-L_{in}(x_2)|\le e^{2\widetilde{K}\Delta t}K_{in}|x_1-x_2|.
\]
The proof is complete.
\end{proof}

 \subsection{Type 1 Methods} We define $\epsilon_1(t) := y^{RQoI}(t)-v(t)$.
 The continuous NINN using Method 1 seen in section \ref{type1methods} for partial observations $\{I_M(u(k\Delta t))\}_{k=0}^{\infty}$ occurring every $\Delta t$ time units is defined as,
\begin{align}\label{contNINNpartial}
\begin{aligned}
\partial_t w &= f_{NN}(w) - \mu \bigchi_{\Omega_{QoI}}(w - [v+\epsilon_1])\qquad\forall t\in(k\Delta t, (k+1)\Delta t), \\ 
w(k\Delta t) &= L(\lim_{t \rightarrow k\Delta t^{-}}w(t)), \\ 
    w(0) &= L_{in}(w_0).
\end{aligned}    
\end{align}
\begin{lemma}
  Assume that Assumptions \ref{def:c_M} and \ref{contResNet} hold. Let $w$ satisfy \eqref{contNINNpartial}, $v$ satisfy \eqref{contNN} and $u$ satisfy \eqref{refdynamics}. Recall Lipschitz constant $K$ from Theorem \ref{thm:LipFNN}. Then $\epsilon_1:= y^{RQoI}(t)-v(t)$ satisfies the following bound for $t\in(k\Delta t, (k+1)\Delta t)$,
 \begin{align}\label{epsilon_bd}
|\epsilon_1| \le K|(I-I_M)\left(L_{out}(w^{-}(k\Delta t))-u(k\Delta t)\right)|\le Kc_M|L_{out}(w^{-}(k\Delta t))-u(k\Delta t)|.
\end{align}
In particular, if we have full observations, which corresponds to $I_M=I$, then $|\epsilon_1|=0$. 
\end{lemma}
\begin{proof}
Let $t\in(k\Delta t, (k+1)\Delta t)$ and let $I_M(u(k\Delta t))$ represents the most 
recent observation. 
We define 
\begin{equation}\label{def:wstar}
w^*=I_M(u(k\Delta t)) + (I-I_M)(L_{out}(w^{-}(k\Delta t))) \in \mathbb{R}^d.
\end{equation}
Next, using this $w^*$ as the input to a continuous ResNet, we can define $y^{RQoI}(t)=F_{NN}(w^*,t-k\Delta t)$. Namely, $y^{RQoI}$ is the forward propagation of $w^*$ through the ResNet (cf.~section~\ref{type1methods} for the discrete setting).
In a similar fashion $v$ solving \eqref{contNN} can be represented by 
$v(t)=F_{NN}(u(k\Delta t),t-k\Delta t)$ for $t\in (k\Delta t, (k+1)\Delta t)$. 
Then by the definition of $\epsilon_1$ and Theorem \eqref{thm:LipFNN},
\begin{align*}
    |\epsilon_1(t)|:=|y^{QoI}(t)-v(t)|&:=|F_{NN}(w^*,t-k\Delta t)-F_{NN}(u(k\Delta t),t-k\Delta t)|\\
    &\le K|w^*-u(k\Delta t)|.
\end{align*}
By inserting the definition of $w^*$ from \eqref{def:wstar} into the last inequality the proof is finished.
\end{proof}
Before, we prove our main result for approximation of Method 1, we introduce an assumption on $\bigchi_{\Omega_{QoI}}$ defined in \eqref{bigchi}.
\begin{assumption}[Stability of $\bigchi_{\Omega_{QoI}}$]\label{def:alpha}
   Let $v$ be as given in \eqref{contNN} and $w$ represents one of the continuous NINNs solutions (\eqref{contNINNpartial} and \eqref{contNINN2partial}). If $\widetilde{w}(t)=v(t)-w(t)$, 
    then we assume that for all $t\in[0,\infty)$, $\exists \, \alpha \ge 0 \mbox{ such that  }|\bigchi_{\Omega_{QoI}}(\widetilde{w}(t))-\widetilde{w}(t)|\le \alpha|\widetilde{w}(t)|$ with $0\le\alpha<1.$ 
\end{assumption}
The above assumption is likely to hold. Let $\alpha = 1$ be the smallest $\alpha$ such that the inequality in 
Assumption~\ref{def:alpha} holds. Since $\alpha$ is smallest, it follows that $|\bigchi_{\Omega_{QoI}}(\widetilde{w}(t^*))-\widetilde{w}(t^*)| = |\widetilde{w}(t^*)|$ for some $t^*\in [0,\infty)$. 
This in turn implies $\bigchi_{\Omega_{QoI}}(\widetilde{w}(t^*)) 
= 0$. The equality holds from the definition of $\bigchi_{\Omega_{QoI}}$ in \eqref{bigchi}. Next, we can infer from $\bigchi_{\Omega_{QoI}}(\widetilde{w}(t^*)) 
= 0$ that $[\widetilde{w}(t^*)]_i=0$ for all $i\in\Omega_{QoI}$ which is unlikely to occur given the presence of neural network error, NINN error and incomplete observations. 
\begin{theorem}\label{thm:method1}
Assume that Assumptions \ref{def:c_M}, \ref{contResNet} and \ref{def:alpha} hold.   Let $\mu$ in \eqref{contNINNpartial} satisfy $\mu> \max\{\frac{2\ln{(2K_{L}^2)}}{(1-\alpha)\Delta t}, \frac{4\widetilde{K}}{1-\alpha}\}$ and $K_L$ be as given in Definition~\ref{def:L}. Then for $u$ and $w$ satisfying \eqref{refdynamics} and \eqref{contNINNpartial} respectively, the following estimate holds
\begin{align*}
    |L_{out}(w^{-}(k\Delta t)) - u(k\Delta t)|^2\le 4\epsilon_{NN}+\frac{8K_{out}^2}{(1-\alpha)^2}\max_{t}|\epsilon_{1}|^2 + 2^{-k}|w_0 - u_0|^2.
\end{align*}
\end{theorem}
\begin{proof}
Define $\widetilde{w}=v-w$ and assume $t\in(k\Delta t,(k+1)\Delta t)$.  Then from \eqref{contNN} 
and \eqref{contNINNpartial} it is easy to see that
\begin{align*}
   \partial_t\widetilde{w} &= f_{NN}(v)-f_{NN}(w) -\mu \bigchi_{\Omega_{QoI}}(\widetilde{w}) - \mu \bigchi_{\Omega_{QoI}}(\epsilon_1).
\end{align*}
Multiply both sides by $\widetilde{w}$,
\begin{align}\label{eq1:continit}
    \frac{1}{2} \partial_t |\widetilde{w}|^2 + \mu|\widetilde{w}|^2 = (f_{NN}(v)-f_{NN}(w),\widetilde{w}) - \mu(\bigchi_{\Omega_{QoI}}(\widetilde{w})-\widetilde{w},\widetilde{w}) - \mu(\bigchi_{\Omega_{QoI}}(\epsilon_1),\widetilde{w}).
\end{align}
Each term on the right hand side is estimated next. Recall \eqref{I_Nineq}, Young's inequality, and we immediately obtain that
\begin{align*}
    |(f_{NN}(v)-f_{NN}(w),\widetilde{w})|&\le |f_{NN}(v)-f_{NN}(w)||\widetilde{w}|
    \le \widetilde{K}|\widetilde{w}|^2.
\end{align*}
\begin{align*}
    \mu|(\bigchi_{\Omega_{QoI}}(\epsilon_1),\widetilde{w})|&\le \mu|\epsilon_1||\widetilde{w}|
    \le \frac{\mu}{2(1-\alpha)}|\epsilon_1|^2+ \frac{(1-\alpha)\mu}{2}|\widetilde{w}|^2.
\end{align*}
Next, using \eqref{def:alpha}, we obtain that
\begin{align*}
    \mu|(\bigchi_{\Omega_{QoI}}(\widetilde{w})-\widetilde{w},\widetilde{w})|&\le \mu |\bigchi_{\Omega_{QoI}}(\widetilde{w})-\widetilde{w}||\widetilde{w}|
    \le \mu\alpha|\widetilde{w}|^2.
\end{align*}
Substituting the above estimates in \eqref{eq1:continit}, we obtain that
\begin{align*}
    \frac{1}{2}\partial_t |\widetilde{w}|^2 + \Big(\frac{(1-\alpha)\mu}{2} - \widetilde{K} \Big)|\widetilde{w}|^2 \le \frac{\mu}{2(1-\alpha)}|\epsilon_1|^2.
\end{align*}
After choosing $\mu$ large enough to satisfy $\frac{(1-\alpha)\mu}{2} - \widetilde{K}>\frac{(1-\alpha)\mu}{4}$ or $\mu > \frac{4\widetilde{K}}{(1-\alpha)}$,
\begin{align}\label{eq1:diffeq1}
    \partial_t |\widetilde{w}|^2 + \frac{(1-\alpha)\mu}{2}|\widetilde{w}|^2 \le \frac{\mu}{(1-\alpha)}|\epsilon_1|^2.
\end{align}
Recall $v$ and $w$ from \eqref{contNN} and \eqref{contNINNpartial} at $t=k\Delta t$. In particular, we have that $\widetilde{w}(k\Delta t) = v(k\Delta t) - w(k\Delta t) = L_{in}(u(k\Delta t)) - L(w^{-}(k\Delta t))$. We use the notation $z^{-}$ to indicate a limit from the left.
Integrating \eqref{eq1:diffeq1}, from $k\Delta t$ to $(k+1)\Delta t$, we obtain that
\begin{align}\label{eq1:diffeq2}
|\widetilde{w}^{-}((k+1)\Delta t)|&^2 \le \frac{2}{(1-\alpha)^2}\max_{t}|\epsilon_1|^2 + e^{-(1-\alpha)\mu\Delta t/2}|\widetilde{w}(k\Delta t)|^2 \\ \nonumber
&= \frac{2}{(1-\alpha)^2}\max_{t}|\epsilon_1|^2 + e^{-(1-\alpha)\mu\Delta t/2}|L(w^{-}(k\Delta t))-L_{in}(u(k\Delta t))|^2.\\ \nonumber
&\le \frac{2}{(1-\alpha)^2}\max_{t}|\epsilon_1|^2 + e^{-(1-\alpha)\mu\Delta t/2}K_{in}^2|L_{out}(w^{-}(k\Delta t))-u(k\Delta t)|^2, 
\end{align}
where in the last step we have used that $L = L_{in} \circ L_{out}$.
 By the triangle inequality we have for all $k$,
\begin{align}\label{triangleinequality1}
    |L_{out}(w^{-}(k\Delta t)) - u(k\Delta t)|&\le K_{out}|\widetilde{w}^{-}(k\Delta t)| + |L_{out}(v^{-}(k\Delta t))-u(k\Delta t)| \\
    &\le K_{out}|\widetilde{w}^{-}(k\Delta t)| + (\epsilon_{NN})^{1/2}. \nonumber
\end{align}
By Young's inequality \eqref{triangleinequality1} becomes,
\begin{align}\label{triangleinequality2}
    |L_{out}(w^{-}(k\Delta t)) - u(k\Delta t)|^2&\le 2K_{out}^2|\widetilde{w}^{-}(k\Delta t)|^2 + 2\epsilon_{NN}.
\end{align}
The first term has been estimated in \eqref{eq1:diffeq2}. The second term is a ResNet error which we have denoted by $\epsilon_{NN}$, see Definition \ref{def:ep_NN}. Choose $\mu$ such that  $e^{-(1-\alpha)\mu\Delta t/2}K_{L}^2<\frac{1}{2}$, recall $K_L=K_{in}K_{out}$ from definition \ref{def:L}. This results in  $\mu>\frac{2\ln{(2K_{L}^2)}}{(1-\alpha)\Delta t}$.
Combining \eqref{triangleinequality2} with \eqref{eq1:diffeq2},
\begin{align*}
    |L_{out}(w^{-}((k+1)\Delta t)) - u((k+1)\Delta t)|^2 &\le 2\epsilon_{NN} + \frac{4K_{out}^2}{(1-\alpha)^2}\max_{t}|\epsilon_{1}|^2\\
    &+ \frac{1}{2}|L_{out}(w^{-}(k\Delta t))-u(k\Delta t)|^2.
\end{align*}
After applying this estimate recursively,
\begin{align*}
    |L_{out}(w^{-}(k\Delta t)) - u(k\Delta t)|^2&\le \frac{1-2^{-k}}{1-2^{-1}}\Big(2\epsilon_{NN}+\frac{4K_{out}^2}{(1-\alpha)^2}\max_{t}|\epsilon_{1}|^2\Big) + 2^{-k}|w_0 - u_0|^2\\
    &\le 4\epsilon_{NN}+\frac{8K_{out}^2}{(1-\alpha)^2}\max_{t}|\epsilon_{1}|^2 + 2^{-k}|w_0 - u_0|^2.
\end{align*}
This concludes the proof.
\end{proof}
  \subsection{Type 2 Methods}

  We remind the reader that the analysis in this section takes into account a system of ResNets, such as in Figure \ref{fig:ResNetsystem}. To account for this, $\mathcal{N}\odot z\in\mathbb{R}^{n_\ell}$ will represent the Type 2 NINN feedback terms for the whole ResNet system. Here $\mathcal{N} = [\mathcal{N}_1, \dots, \mathcal{N}_{N_R}]$ with $\mathcal{N}_i \in \mathbb{R}$ and $z = [z_{\ell_1}, \dots, z_{\ell_{N_R}}]$ with $z_{\ell_i} \in \mathbb{R}^{n_{\ell_i}}$. 
  Moreover $\mathcal{N}\odot z$ is defined as $\mathcal{N}\odot z := [\mathcal{N}_{1}z_{\ell_1} \dots \mathcal{N}_{{N_R}}z_{\ell_{N_R}}]\in\mathbb{R}^{n_\ell}$ where $N_R$ is the number of ResNets in the system and $n_\ell = \sum_{i=1}^{N_R} n_{\ell_i}$.
 If ResNet $i$ does not have a NINN feedback term then we assume $\mathcal{N}_{i} z_{\ell_i} =0\in\mathbb{R}^{n_{\ell_i}}$.
  Otherwise, we will assume $\mathcal{N}_{i}\in\mathbb{R}$ is given by the second option in Case 1 (or Case 2).  
  To be more specific, we will write
  \[
  N(t) = I_M(L_{out}(F_{NN}(w(t),(k+1)\Delta t - t))-u(k\Delta t)).
  \]
 Using Assumption \ref{def:c_M} we can estimate $|\mathcal{N}|$,
\begin{align}\label{Nell} \nonumber
|\mathcal{N}(t)| &= |I_M(L_{out}(F_{NN}(w(t),(k+1)\Delta t - t))-u(k\Delta t))|\\
&\le (c_M+1)|L_{out}(F_{NN}(w(t),(k+1)\Delta t - t))-u(k\Delta t)|\\
&= c_{\mathcal{N}}|L_{out}(F_{NN}(w(t),(k+1)\Delta t - t))-u(k\Delta t)|, \nonumber
\end{align}
for $t\in(k\Delta t, (k+1)\Delta t)$. By rescaling, $z$ satisfies $|z|=1$. 
  We define the continuous NINN using Method 2 described in section \ref{type2methods} with partial observations $\{I_M(u(k\Delta t))\}_{k=0}^{\infty}$ occurring every $\Delta t$ time units,
\begin{align}\label{contNINN2partial}
\partial_t w &= f_{NN}(w) - \mu \bigchi_{\Omega_{QoI}}(\mathcal{N}\odot z), \qquad\forall t\in(k\Delta t, (k+1)\Delta t), \\ \nonumber
w(k\Delta t) &= L(\lim_{t \rightarrow k\Delta t^{-}}w(t)), \\ \nonumber
    w(0) &= L_{in}(w_0).
\end{align}
Next we add and subtract $\widetilde{w}=v-w$ to \eqref{contNINN2partial} and define $\epsilon_2:=\mathcal{N}\odot z + \widetilde{w}$. This gives us the alternate formulation of \eqref{contNINN2partial},
\begin{align}\label{contNINN2partialalt}
\partial_t w &= f_{NN}(w) - \mu \bigchi_{\Omega_{QoI}}(w - [v+\epsilon_2]), \qquad\forall t\in(k\Delta t, (k+1)\Delta t), \\ \nonumber
w(k\Delta t) &= L(\lim_{t \rightarrow k\Delta t^{-}}w(t)), \\ \nonumber
    w(0) &= L_{in}(w_0).
\end{align}
Note the similarity between the Type 1 Method formulation \eqref{contNINNpartial} and the Type 2 Method alternate formulation \eqref{contNINN2partialalt}.
We will need the following metric $G$ in the proof for Method 2 which represents the cost of translating between the state space of the dynamics and the state space of the ResNet (see Figure~\ref{diag:online}).
  \begin{definition}[G-metric]\label{def:G}
  Let $G:=\max_{k} |L_{out}\circ L_{in}(u(k\Delta t))-u(k\Delta t)|$ with $k\in\mathbb{Z}^+$. 
  \end{definition} 
  %

\begin{lemma} Assume that Assumptions \ref{def:c_M} and \ref{contResNet} hold. Let $w$ satisfy \eqref{contNINN2partial}, $v$ satisfy \eqref{contNN} and $u$ satisfy \eqref{refdynamics}. Recall Lipschitz constant $K_{out}$ from Definition \ref{def:L}.
Then $\epsilon_2:=\mathcal{N}\odot z + \widetilde{w}$ satisfies the following bound for $t\in(k\Delta t, (k+1)\Delta t)$,
\begin{align}\label{epsilon_bd2}
    |\epsilon_2|\le |\mathcal{N}|\Big(1+\frac{1}{c_{\mathcal{N}}K_{out}}\Big) + \epsilon_{2,w}(\Delta t) + \frac{G}{K_{out}},
\end{align}
where the error term $\epsilon_{2,w}\to 0$ as $\Delta t\to 0$.
\end{lemma}
\begin{proof}
Recall $|z|=1$ which implies $|\mathcal{N}\odot z|\le |\mathcal{N}|$. Then we have
\begin{align}\label{Nellest1} \nonumber
    |\epsilon_2| &= |\mathcal{N}\odot z+\widetilde{w}|\\ \nonumber
    &\le \big|\mathcal{N}\odot z + \frac{|\mathcal{N}\odot z|\widetilde{w}}{c_{\mathcal{N}}K_{out}|\widetilde{w}|}\big| + \big|\frac{|\mathcal{N}\odot z|\widetilde{w}}{c_{\mathcal{N}}K_{out}|\widetilde{w}|}-\widetilde{w}\big|\\
    &\le |\mathcal{N}|(1+\frac{1}{c_{\mathcal{N}}K_{out}}) + \frac{||\mathcal{N}\odot z|-c_{\mathcal{N}}K_{out}|\widetilde{w}||}{c_{\mathcal{N}}K_{out}}.
\end{align}
The second term in the above estimate is handled next. Let $t\in(k\Delta t,(k+1)\Delta t)$ and recall the estimate \eqref{Nell}. 
\begin{align*}
    |(\mathcal{N}\odot z)(t)|\le|\mathcal{N}(t)| &\le c_{\mathcal{N}}|L_{out}(F_{NN}(w(t),(k+1)\Delta t - t))-u(k\Delta t)| \\
    &= c_{\mathcal{N}} |L_{out}(w)-L_{out}(w) + L_{out}(v)-L_{out}(v) \\
    &\qquad + L_{out}(F_{NN}(w(t),(k+1)\Delta t - t))-u(k\Delta t)| \\
    &\le c_{\mathcal{N}}K_{out}|F_{NN}(w(t),(k+1)\Delta t - t) - w(t)| \\
    &\qquad + c_{\mathcal{N}}|L_{out}(v)-u(k\Delta t)|+ c_{\mathcal{N}}  K_{out}|\widetilde{w}| \\
    &\le c_{\mathcal{N}}K_{out}|F_{NN}(w(t),(k+1)\Delta t - t) - w(t)|\\
    &\qquad + c_{\mathcal{N}}|L_{out}(v)-L_{out}\circ L_{in}(u(k\Delta t))|\\
    & \qquad +c_{\mathcal{N}}|L_{out}\circ L_{in}(u(k\Delta t))-u(k\Delta t)| + c_{\mathcal{N}}K_{out}|\widetilde{w}|.
\end{align*}
This implies
\begin{align*}
    \big||\mathcal{N}\odot z|-c_{\mathcal{N}}K_{out}|\widetilde{w}|\big| &\le c_{\mathcal{N}}K_{out}|F_{NN}(w(t),(k+1)\Delta t - t) - w(t)|\\
    &+ c_{\mathcal{N}}K_{out}|v-L_{in}(u(k\Delta t))|\\
    &+c_{\mathcal{N}}|L_{out}\circ L_{in}(u(k\Delta t))-u(k\Delta t)|.
\end{align*}
Notice the first two terms on the right hand side tend towards $0$ as $\Delta t\to 0$. 
To see this, the first term is equal to 
\[
c_{\mathcal{N}}K_{out}|F_{NN}(w(t),(k+1)\Delta t - t) - w(t)| = c_{\mathcal{N}}K_{out}\bigg|\int_0^{(k+1)\Delta t -t}f_{NN}ds\bigg|,
\]
by \eqref{eq:FNN}. Then recall that $t \in (k\Delta t, (k+1)\Delta t)$. Thus the above right-hand-side goes to zero as $\Delta t \rightarrow 0$. 
Similarly, the second term is equal to
\begin{align*}
c_{\mathcal{N}}K_{out}|v(t)-L_{in}(u(k\Delta t))| &= c_{\mathcal{N}}K_{out}|F_{NN}(L_{in}(u(k\Delta t)),t-\Delta t)-L_{in}(u(k\Delta t))|\\
& = c_{\mathcal{N}}K_{out}\bigg|\int_0^{t-\Delta t}f_{NN}ds\bigg| ,
\end{align*}
which also goes to zero as $\Delta t \rightarrow 0$. 
Recall $f_{NN}$ is Lipschitz by Assumption \ref{contResNet}, therefore it is continuous and both integrals converge to 0 as
$\Delta t \to 0$. Let us return to estimate \eqref{Nellest1} and recall $G$ from Definition~\ref{def:G}.
Then an appropriate bound is
\begin{align}\label{epsilon_bd3}
    |\epsilon_2|\le |\mathcal{N}|\Big(1+\frac{1}{c_{\mathcal{N}}K_{out}}\Big) + \epsilon_{2,w}(\Delta t) + \frac{G}{K_{out}}.
\end{align}
This concludes the proof.
\end{proof}
The following Theorem uses the same proof from Theorem \ref{thm:method1} by using formulation 
\ref{contNINN2partialalt}.
\begin{theorem}\label{thm:method2}
Assume that Assumptions \ref{def:c_M}, \ref{contResNet} and \ref{def:alpha} hold.   Let $\mu$ in \eqref{contNINN2partial} satisfy $\mu> \max\{\frac{2\ln{(2K_{L}^2)}}{(1-\alpha)\Delta t}, \frac{4\widetilde{K}}{1-\alpha}\}$. $K_L$ is given in Definition~\ref{def:L}. Then for $u$ and $w$ satisfying \eqref{refdynamics} and \eqref{contNINNpartial} respectively, the following estimate holds
\begin{align*}
    |L_{out}(w^{-}(k\Delta t)) - u(k\Delta t)|^2\le 4\epsilon_{NN}+\frac{8K_{out}^2}{(1-\alpha)^2}\max_{t}|\epsilon_{2}|^2 + 2^{-k}|w_0 - u_0|^2.
\end{align*}
\end{theorem}
\subsubsection{Error Summary}
The estimates in Theorems~\ref{thm:method1} and \ref{thm:method2} consist 
of three terms. The last term exhibits the typical exponential convergence. 
The first term, $\epsilon_{NN}$ (Definition \ref{def:ep_NN}), is expected to be small in practice 
as the user has freedom over which ResNet is used. Finally, the middle terms consist of $\epsilon_1$ 
and $\epsilon_2$. Notice that, $\epsilon_1$ is the difference between $y^{RQoI}$ and $v$. Recall 
the estimate \eqref{epsilon_bd} which suggests $\epsilon_1$ is minimized with an increase in 
observation of $u$. Moreover, $\epsilon_2 = \mathcal{N}_\ell\mu_W+\widetilde{w}$ is estimated 
by \eqref{epsilon_bd2}. This error can be minimized by choosing ResNets with $\Delta t$ and $G$ small. 
On the other hand, the bound for $\epsilon_2$ in Method 2 is more tedious (see \eqref{epsilon_bd2}). 
However, we will see in the numerical section that Method 2 performs better for our experiments.

\subsection{Discrete Dynamics and Convergence Estimates}

In the previous section we proved convergence error results for the continuous NINNs described in 
\eqref{contNINNpartial} and \eqref{contNINN2partial}. Notice, that the discrete NINNs have ResNet 
structure \eqref{ResNet}. To derive the estimates in the discrete setting, we can simply use a triangle 
inequality. Let $u$ represent the true solution \eqref{refdynamics}, $w_{cont}$ represent the continuous 
NINNs (obtained using Method 1 or 2), and let $w$ represent the discrete NINNs. Then for all $k\in\mathbb{Z}^+,$
\begin{align*}
    \|w(k\Delta t)-u(k\Delta t)\| \le \|w(k\Delta t)-w_{cont}(k\Delta t)\| + \|w_{cont}(k\Delta t)-u(k\Delta t)\|.
\end{align*}
The first term on the right-hand-side is the time discretization error for DNN, for instance, forward Euler. The 
approximation of the second term has been discussed in Theorems~\ref{thm:method1} and \ref{thm:method2}.
\section{Experimental Introduction}
The purpose of this section is to set a stage for our numerical experiments. 
In section \ref{sec:DA}, NINNs are implemented as data assimilation algorithms, 
where we test the algorithms on the Lorenz 63 and 96 ODE systems. In section 
\ref{sec:chem}, NINNs aid ResNets in replicating stiff ODEs arising from 
chemically reacting flows.

\subsection{ResNet Training}\label{sec:ResNettraining}
The ResNets in sections~\ref{sec:DA} and \ref{sec:chem} are trained using the 
specifications in this section. The systems will be comprised of multiple ResNets 
where each ResNet has output in $\mathbb{R}$ and is responsible for one 
component of the output. Figure \ref{fig:ResNetsystem} is one such example for 
a system with input/output in $\mathbb{R}^3$. Additionally, the ResNet systems 
will be trained on 15,000 training samples. The training samples are input/output 
pairs generated from a dynamical system corresponding to a specific time step. 
In section \ref{sec:DA}, the time step size for the Lorenz ODEs is $10^{-2}$ and 
in section \ref{sec:chem} the time step size is $5\cdot 10^{-8}$ for the chemically 
reacting flow ODEs. The training samples will be split 80-20, i.e. 80\% of the 
samples will be used for training and 20\% will be used for validation. A patience 
of 400 iterations is used with the training data. The latter means that if the validation 
error increases 
then training will continue for 400 more iterations. The BFGS optimization routine 
is used in conjunction with bias ordering, see section \ref{sec:biasordering} for 
the details. The parameters are initialized with box initialization \cite{box}.

\subsection{Data Assimilation Protocol}\label{sec:DAP}
In section \ref{sec:DA}, ResNets are trained to learn the Lorenz 63 and 96 ODE 
systems. We generate synthetic partial observations and, by equipping the ResNets 
with NINNs, recover the solution corresponding to the partial observations. Therefore 
showing NINNs are effective as data assimilation algorithms. The synthetic partial 
observations are generated as follows. Using the Lorenz ODEs, we compute 
$N^{f}=100$ reference solutions by generating random initial conditions from a 
Gaussian random variable with a mean of 0 and a standard deviation of 10. The 
initial conditions are evolved forward in time to 110 time units for the Lorenz 63 ODES and 120 time units for the Lorenz 96 ODEs. We then extract 
observations every $10^{-1}$ time units starting at 100 time units. The observations 
are then used as QoI for the NINNs and we compute the $N^f=100$ corresponding 
NINNs solutions. As an evaluation metric we will use the spatio-temporal root mean 
square error (RMSE): 
\begin{equation}\label{RMSE}
RMSE = \sqrt{\frac{1}{(K-k_0)\times N^f}\sum_{k=k_0}^K\sum_{n=1}^{N^f}\big(x_{n,k}^{\text{ALG}}-x_{n,k}^{\text{ref}}\big)^2}.
\end{equation}
RMSE compares the reference solution to the corresponding algorithm 
solution every $10^{-1}$ time units. Here $n$ corresponds to the $N^f$ 
different solutions and $k$ corresponds to time. Therefore $x_{n,k}^{ALG}$ 
is the $n$-th algorithm solution at time $k\cdot 10^{-1}$ and similarly 
$x^{ref}_{n,k}$ is the $n$-th reference solution at time $k\cdot 10^{-1}$. 
We choose $k_0$ such that the metric is evaluated after 5 time units of 
the data assimilation process has passed, i.e. $k_0=50$. Similarly, $K$ corresponds to the end of the evaluation period. For the Lorenz 63 ODEs we evaluate over 5 time units corresponding to $K=100$. For the Lorenz 96 ODEs we evaluate over 15 time units corresponding to $K=200$.

\subsection{Exponential Decay}\label{sec:expdecay}
As stated in section \ref{sec:ResNettraining}, the ResNets in section \ref{sec:DA} are given training samples corresponding to a time step size of $10^{-2}$. As stated in section \ref{sec:DAP}, the synthetic observations are available every $10^{-1}$ time units. Therefore, NINNs compute 10 ResNet evaluations before the observations are updated again. During this time, the observation becomes outdated as more ResNet evaluations are done, advancing in time. We dampen the $\mu$ parameter between observations with $e^{-i\Lambda}$ for $i=0,1,...,9$, i.e. $\mu \to \mu e^{-i\Lambda}$.  For the experiments in section \ref{sec:DA} we use an exponential decay factor of $\Lambda=1/5$, $\Lambda=1$ or $\Lambda=3$. 

\subsection{Benchmark Algorithms}
In section \ref{sec:DA}, we will compare the NINNs data assimilation 
algorithms against two other data assimilation algorithms. The first will 
be standard nudging introduced in section \ref{sec:Generalnudging}. 
As the second, we introduce the Direct Observation Algorithm. 
Pictured in Figure \ref{fig:DirectObs} are the two steps of the Direct 
Observation Algorithm. 
Here ResNet input/output is in $\mathbb{R}^3$ and QoI are available 
for the second component. The ResNet inputs are updated with 
the QoI when available. This is in stark contrast to NINNs which 
nudge the ResNet towards the QoI. By comparing RMSE values between NINNs and 
this algorithm, we will demonstrate that the NINNs are doing 
something more than forcing components to be equal to the QoI. 
This approach is referred to in the tables below as Direct Obs.


\pgfdeclarelayer{background}
\pgfdeclarelayer{foreground}
\pgfsetlayers{background,main,foreground}
\tikzstyle{mybox} = [draw=black, fill=blue!30, very thick,
    rectangle, rounded corners, inner sep=1pt, inner ysep=15pt]
  \begin{figure}[htb]
  \centering
{\footnotesize
\begin{tikzpicture}[
dot/.style = {circle, fill, minimum size=#1,
              inner sep=0pt, outer sep=0pt},
dot/.default = 6pt  
]
\draw [black, thick] (0,0) to [square left brace ] (0,2);
\draw [black, thick] (.5,0) to [square right brace] (.5,2);

\draw [black, thick] (4.1,0) to [square left brace ] (4.1,2);
\draw [black, thick] (4.6,0) to [square right brace] (4.6,2);

\draw [black, thick] (6,0) to [square left brace ] (6,2);
\draw [black, thick] (6.5,0) to [square right brace] (6.5,2);
\draw [black, thick] (10.1,0) to [square left brace ] (10.1,2);
\draw [black, thick] (10.6,0) to [square right brace] (10.6,2);

\node [mybox] (box2) at (2.3,1) {%
    \begin{minipage}{50pt}\centering
        ResNet 
    \end{minipage}
};
\node [mybox] (box2) at (8.3,1) {%
    \begin{minipage}{50pt}\centering
        ResNet 
    \end{minipage}
};

\node [text width=2.5cm](ResNetSystem) at (.25,-.39) {\footnotesize Insert QoI into initial condition};
\node (ResNetSystem) at (4.35,-.35) {\footnotesize Output};
\node [text width=2cm](ResNetSystem) at (6.35,-.35) {\footnotesize Update Output with QoI};

\draw[dotted,->, line width=1.5pt] (.75,1) -- (1.25,1);
\draw[dotted,->, line width=1.5pt] (3.35,1) -- (3.85,1);
\draw[dotted,->, line width=1.5pt] (6.75,1) -- (7.25,1);
\draw[dotted,->, line width=1.5pt] (9.35,1) -- (9.85,1);

\node at (.25,1.75) {$x_0$};
\node[text=red] at (.25,1) {$y^{QoI}_0$};
\node at (.25,.25) {$z_0$};

\node at (4.35,1.75) {$x_1$};
\node at (4.35,1) {$y_1$};
\node at (4.35,.25) {$z_1$};

\node at (6.25,1.75) {$x_1$};
\node at (6.25,.25) {$z_1$};
\node[text=red] at (6.25,1) {$y^{QoI}_1$};

\node at (10.35,1.75) {$x_2$};
\node at (10.35,1) {$y_2$};
\node at (10.35,.25) {$z_2$};

\end{tikzpicture}
}
 \caption{Two steps of the Direct Observation Algorithm with QoI available for the second component. The ResNets are operating without NINNs.}
\label{fig:DirectObs}
\end{figure}
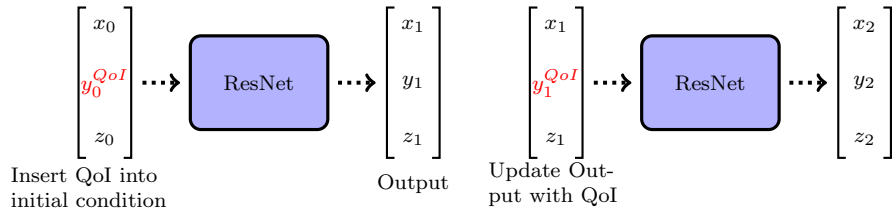
\section{Data Assimilation}\label{sec:DA}
In this section we apply NINNs to ResNets which have learned the Lorenz 63 and 96 ODEs. We will refer to the Type 1 Method described in section \ref{type1methods} as NINN \#1. We will refer to the Type 2 Method from section \ref{type2methods} with Case 1, where $N_\ell=W_{L-1}f_{L}\circ\dots\circ f_{\ell+1}(y_\ell) - y^{QoI}$, as NINN \#2. 
\subsection{Lorenz 63}\label{sec:experiments63}
The Lorenz 63 model is given by the three coupled ODEs 
\begin{subequations}
\begin{align}\label{lorenz}
    d_t x &= \sigma(y-x),\\
    d_t y &= x(\rho-z)-y,\\
    d_t z &= xy-\beta z. \label{lorenz3}
\end{align}
\end{subequations}
We set $\sigma=10,\ \beta=8/3,\ \rho=28,$ which is known to exhibit chaotic behavior \cite{lorenz}. 
The equations \eqref{lorenz}--\eqref{lorenz3} are solved using an explicit Runge-Kutta (4,5) in MATLAB. 
We will focus on two types of observations, $x$ component only and $y$ component only observations. 
After training ResNets with various configurations widths and hidden layers we make the following 
distinction in the RMSE tables. The best case scenario is the lowest RMSE result from all of the 
ResNets. ResNet \#1 and ResNet \# 2 are distinct ResNets picked from the various ResNets trained. 
In this case, the ResNets used in the best case scenario have 8 or 10 hidden layers with a width of 15.  
ResNet \#1 has 6 hidden layers with a width of 50. ResNet \#2 has 3 hidden layers with a width of 50. 
The optimal nudging parameter $\mu$ is different for each case (observation type and ResNet 
combination) and is searched  manually as best as possible. We observe that NINN \#2 is outperforming 
all of the algorithms for each of the three scenarios except the Direct Obs algorithm in the best 
case scenario for y-observations. However, we see that for ResNet \# 2 the Direct Obs algorithm 
is unable to compute solutions while NINN \#2 is still able to obtain low RMSE values. We can 
conclude from this that nudging the solution towards the observations is preferred instead of directly 
inserting observations. In Figure~\ref{fig:sidebyside}, we compute solutions for the two NINNs on 
a randomly generated reference solution that is outside of the training set. 

\begin{table}[!htb]
\begin{center}
\scalebox{0.8}{
\begin{tabular}{ |l|l|l|l| }
\hline
\multicolumn{4}{ |c| }{RMSE-Lorenz 63} \\
\hline
 &Method& $x$-obs & $y$-obs \\ \hline
\multirow{4}{*}{Best Case Scenario} & Nudging & 6.0782 & 5.7953 \\
 & NINN \#1 & 5.9572 & 3.8900 \\
 & NINN \#2 & 4.0021 & 2.0365 \\
 & Direct Obs & 6.6130 & 1.1955 \\ \hline
\multirow{3}{*}{ResNet \#1} & NINN \#1 & 19.8587 & 11.1817 \\
 & NINN \#2 & 5.5087 & 2.2586 \\
 & Direct Obs & 8.5042 & 3.4967 \\ \hline
\multirow{3}{*}{ResNet \#2} & NINN \#1 & Inf & Inf \\
 & NINN \#2 & 20.8884& 2.8645 \\
 & Direct Obs & Inf & Inf \\ \hline
\hline
\end{tabular}}
\end{center}
\caption{Lorenz 63 RMSE values calculated on 100 reference solutions over 5 time units starting 5 time units after the data assimilation process begins.}
\label{table00}
\end{table}

\begin{figure}[ht]
    \centering
    \includegraphics[width=0.495\textwidth]{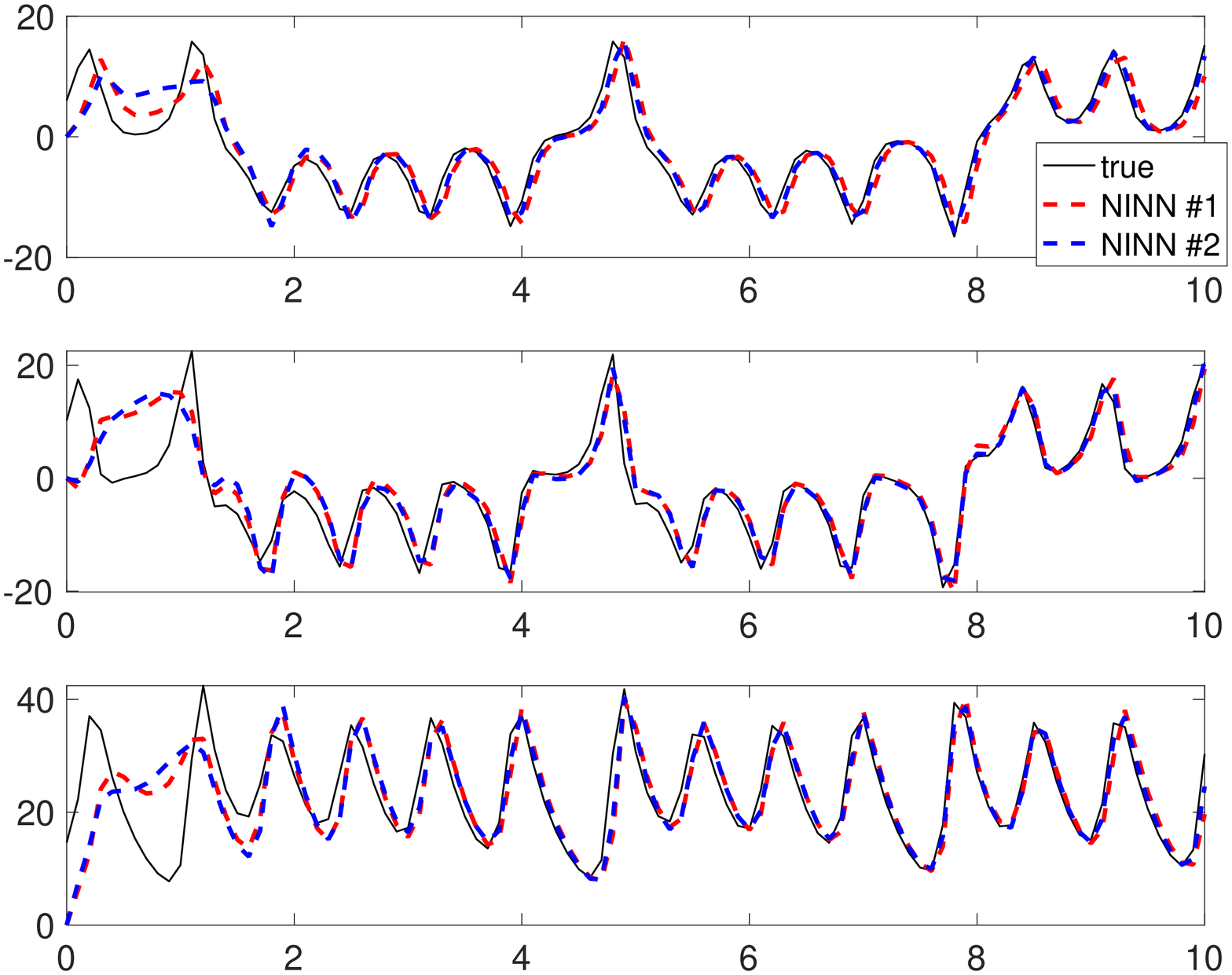}
    \includegraphics[width=0.495\textwidth]{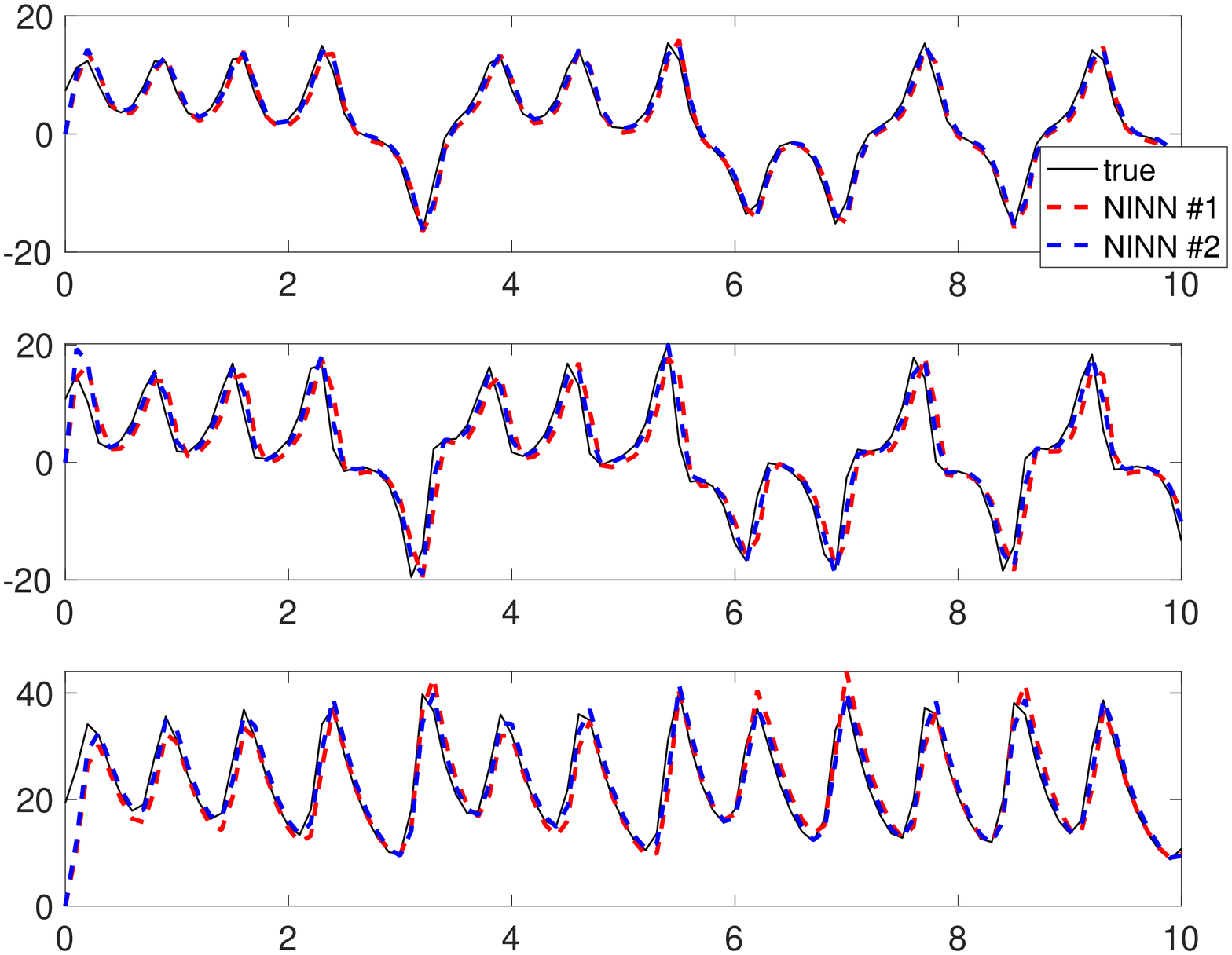}
    \caption{Left: $x$-component observations. Right: $y$-component observations. From top to bottom the $x,y,z$ components of the reference (true), NINN \# 1, and NINN \# 2 solutions over 10 time units for the Lorenz 63 model. }
    \label{fig:sidebyside}
\end{figure}
\subsection{Lorenz 96}
The Lorenz 96 model is given by the following set of ODEs: 
\begin{align}\label{Lorenz96}
    d_t x_i = (x_{i+1} - x_{i-2})x_{i-1} - x_i + F,\quad i=1,2,...,40,\\
    x_{-1} = x_{39}, \quad x_0 = x_{40}, \quad \text{ and }\quad  x_{41} = x_1.
\end{align}
We set $F=10$ which is known to exhibit chaotic behavior \cite{Lorenz96}. We follow a similar procedure as with the Lorenz 63 model and we consider two different observation patterns. We observe approximately $33\%,\text{and } 50\%$ of the state which corresponds to observing $13\text{ and }20$ components, respectively. The ResNets are trained using the same setup from the Lorenz 63 section with one major difference. We still train a ResNet for each of the 40 components but instead of training ResNets that take in inputs in $\mathbb{R}^{40}$, we use the structure of \eqref{Lorenz96} to reduce the input for the ResNets to $\mathbb{R}^4$. The ResNet corresponding to the $i$-th component takes in as input the $i-2,i-1,i,i+1$ components. We refer to these ResNets as being reduced as their input size is 4 compared to the state space size of 40. For the best case scenario we used a reduced ResNet with 9 hidden layers with a width of 15. ResNet \#1 is reduced and has 8 hidden layers with a width of 15. ResNet \#2 has 6 hidden layers with a width of 50. In Table~\ref{table11}, we calculate RMSE values as described above in section \ref{sec:DAP} over the time interval 5 to 20 time units. We observe NINN \#2 is performing the best in each of the three different scenarios. Again, we can conclude that nudging the solution towards the observations is preferred instead of directly inserting observations. In Figure~\ref{fig:sidebyside2}, we compute solutions for the two NINNs on a randomly generated reference solution. We observe in both plots NINN \#2 is tracking the solution more reliably.

\begin{table}[!htb]
\begin{center}
\scalebox{0.8}{
\begin{tabular}{ |l|l|l|l| }
\hline
\multicolumn{4}{ |c| }{RMSE-Lorenz 96} \\
\hline
 &Method& 20-obs & 13-obs \\ \hline
\multirow{4}{*}{Best Case Scenario} & Nudging & 11.9757 & 25.1511 \\
 & NINN \#1 & 24.1348 & 31.1058 \\
 & NINN \#2 & 9.7400 & 24.4759 \\
 & Direct Obs & 11.0766 & 25.8542 \\ \hline
\multirow{3}{*}{ResNet \#1} & NINN \#1 & 24.1348 & 31.1058 \\
 & NINN \#2 & 10.2588 & 28.0722 \\
 & Direct Obs & Inf & Inf \\ \hline
\multirow{3}{*}{ResNet \#2} & NINN \#1 & 29.6423 & 35.1599 \\
 & NINN \#2 & 21.1354 & 32.7238 \\
 & Direct Obs & 23.0963 & Inf \\ \hline
\hline
\end{tabular}}
\end{center}
\caption{Lorenz 96 RMSE values calculated on 100 reference solutions over 15 time units starting 5 time units after the data assimilation process begins.}
\label{table11}
\end{table}
\begin{figure}[ht]
    \centering
    \includegraphics[width=0.495\textwidth]{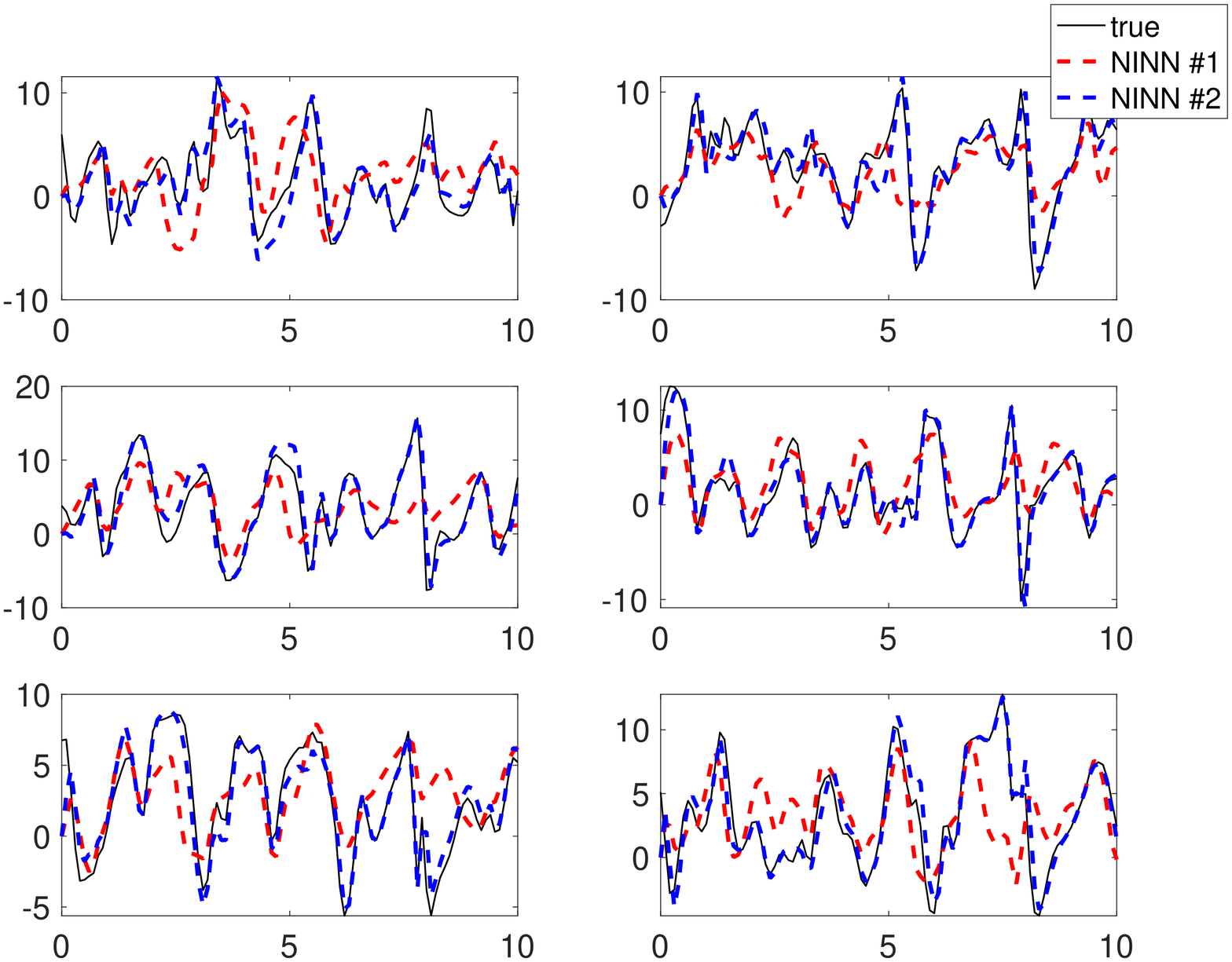}
    \includegraphics[width=0.495\textwidth]{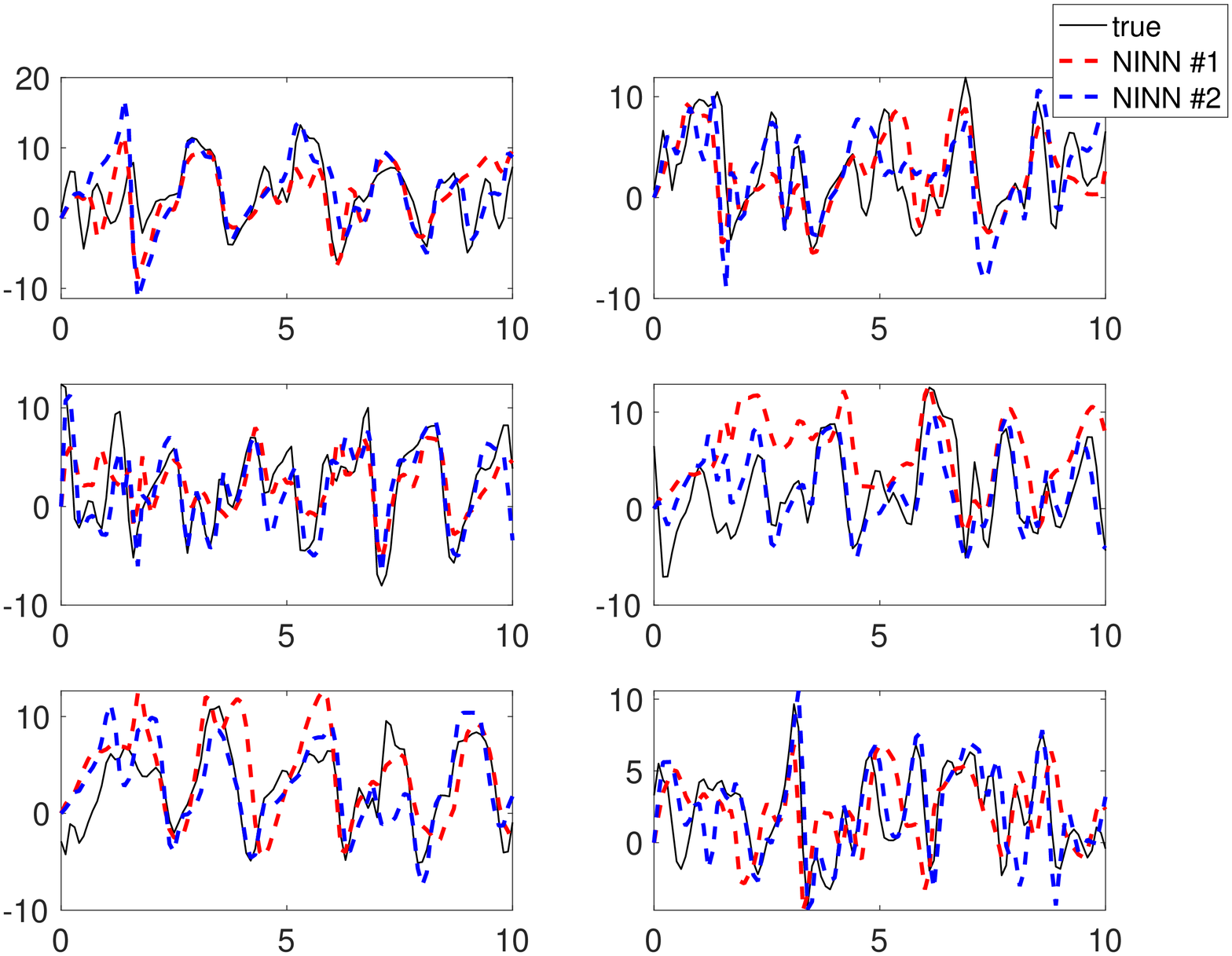}
    \caption{Left: 20 observations on even components. Right: 13 observations on every third component. From left to right, top to bottom the first 6 components of the reference (true), NINN \#1, and NINN \#2 solutions over 10 time units for the Lorenz 96 model.}
    \label{fig:sidebyside2}
\end{figure}
\section{Chemical Kinetics}\label{sec:chem}
The purpose of this section is to demonstrate the effectiveness of NINNs in improving pre-existing neural networks. In particular we show how NINNs can be used to aid neural networks designed to learn ODEs describing chemically reacting flows. We present an experiment where ResNets learn one time step of a stiff ODE modeling a reduced $H_2-O_2$ reaction. The model tracks the reactions of eight species and temperature over time. For more information on approximating the model with ResNets see the recent works \cite{antil2021deep,brown2021novel}. The training data for the ResNets is generated by CHEMKIN \cite{chemkin}. For each species and temperature there corresponds a ResNet for a total of 9 ResNets. The training data is generated from initial conditions with an equivalence ratio of one and 100 temperatures varying from 1300 to 2500 Kelvin. 
The ResNets used in this example have 7 hidden layers with a width of 30. This particular ResNet system is able to capture the flow well for temperatures above 1600K and struggles with capturing the flow for temperatures below 1600K with 1300K being the worst.
To improve the ResNet accuracy we can introduce NINNs. As we are given the exact initial data, we can nudge each of the 9 ResNets towards the initial data for a period of time initially. This nudge  corrects the ResNet in the lower temperature region. See Figures \ref{fig:sidebyside3} and \ref{fig:sidebyside4}. For this example we used the Case 1, Type 2 Method with $N_\ell=W_{L-1}f_{L}\circ\dots\circ f_{\ell+1}(y_\ell) - y^{QoI}$ from section \ref{type2methods}.
\begin{figure}[ht]
    \centering
    \includegraphics[width=0.495\textwidth]{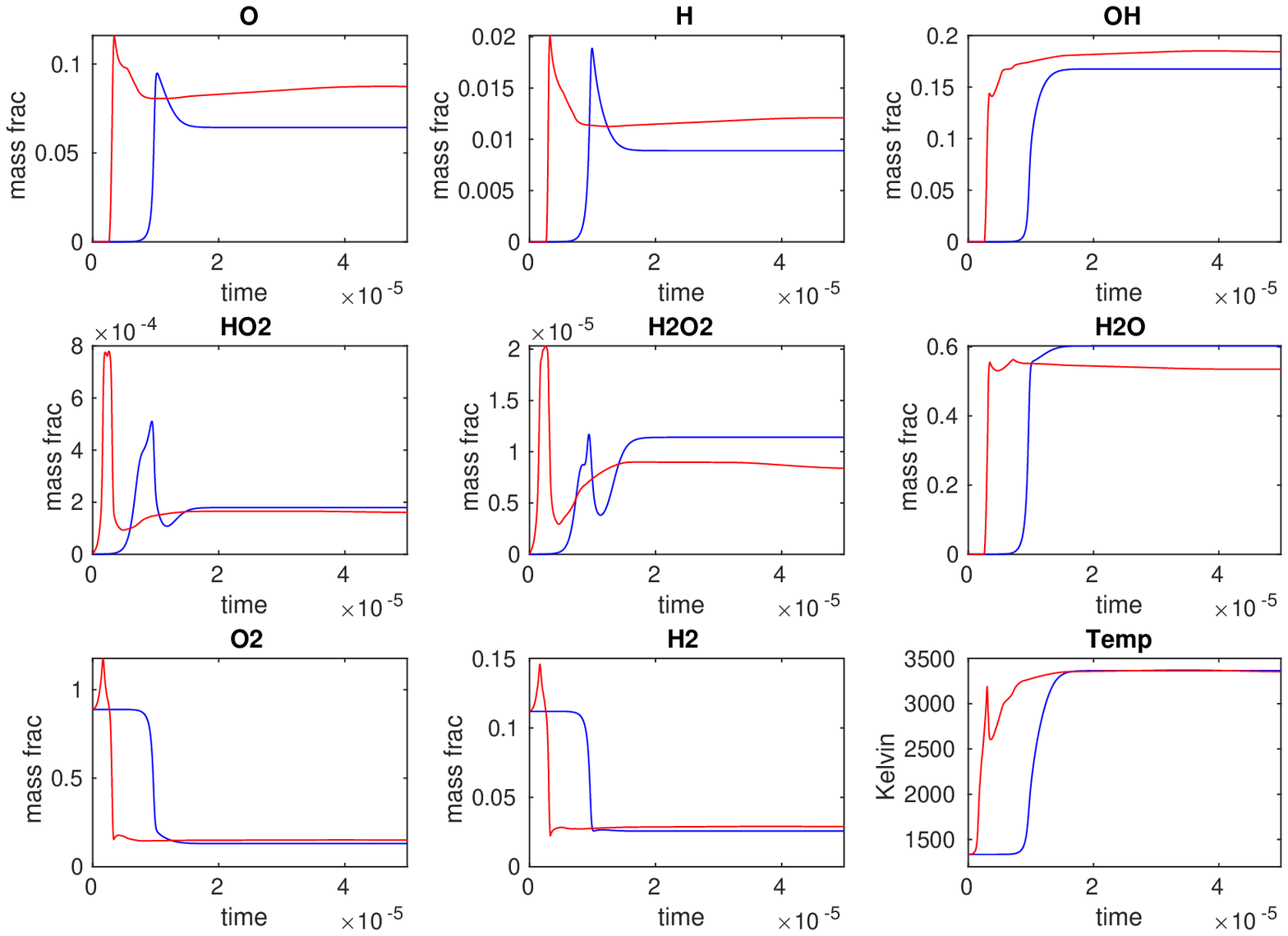}
    \includegraphics[width=0.495\textwidth]{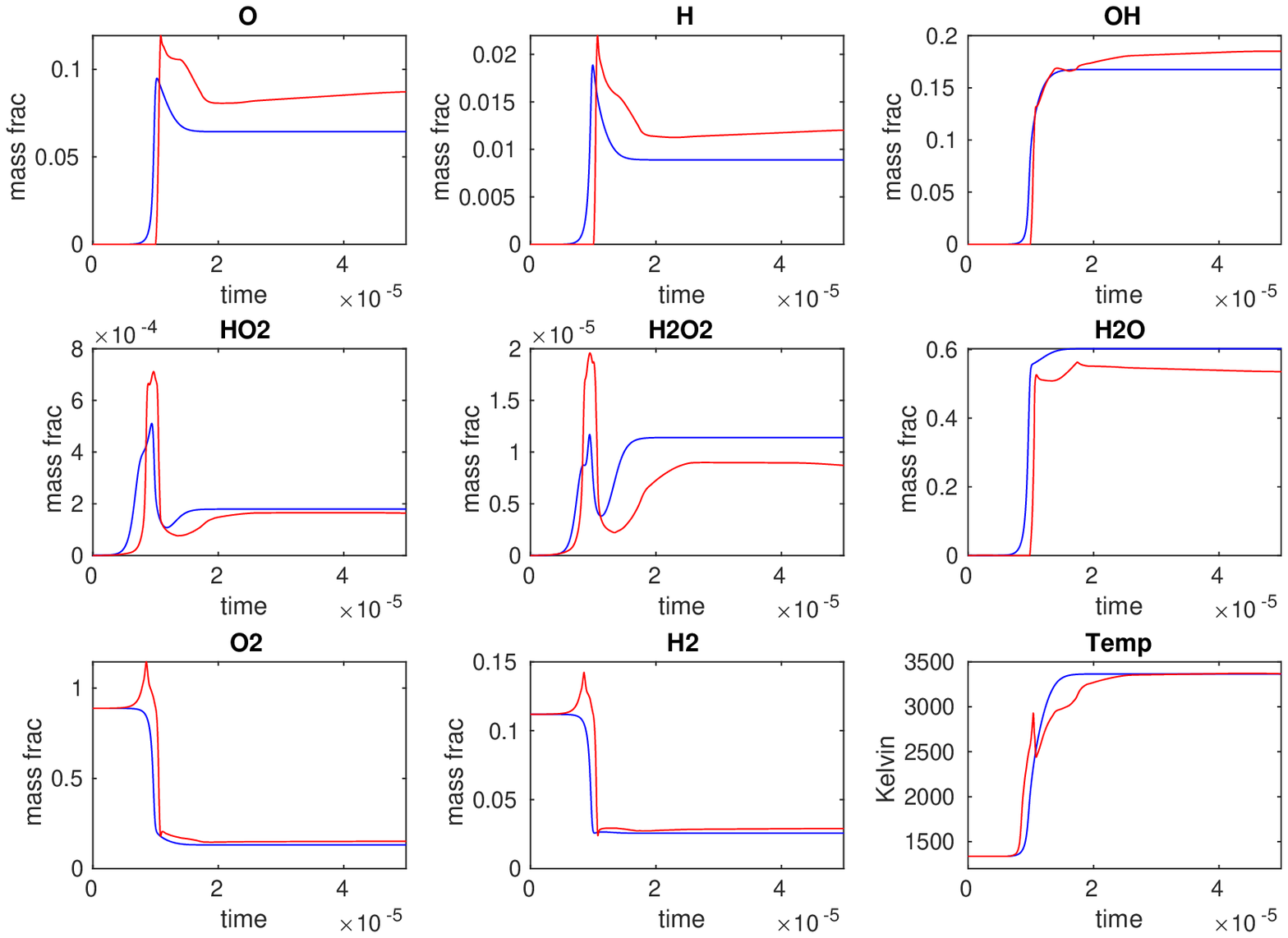}
    \caption{Left: No nudging. Right: Nudging towards initial data. Initial temperature 1336K. Blue is the CHEMKIN solution and red is the ResNet solution.}
    \label{fig:sidebyside3}
\end{figure}

\begin{figure}[ht]
    \centering
    \includegraphics[width=0.495\textwidth]{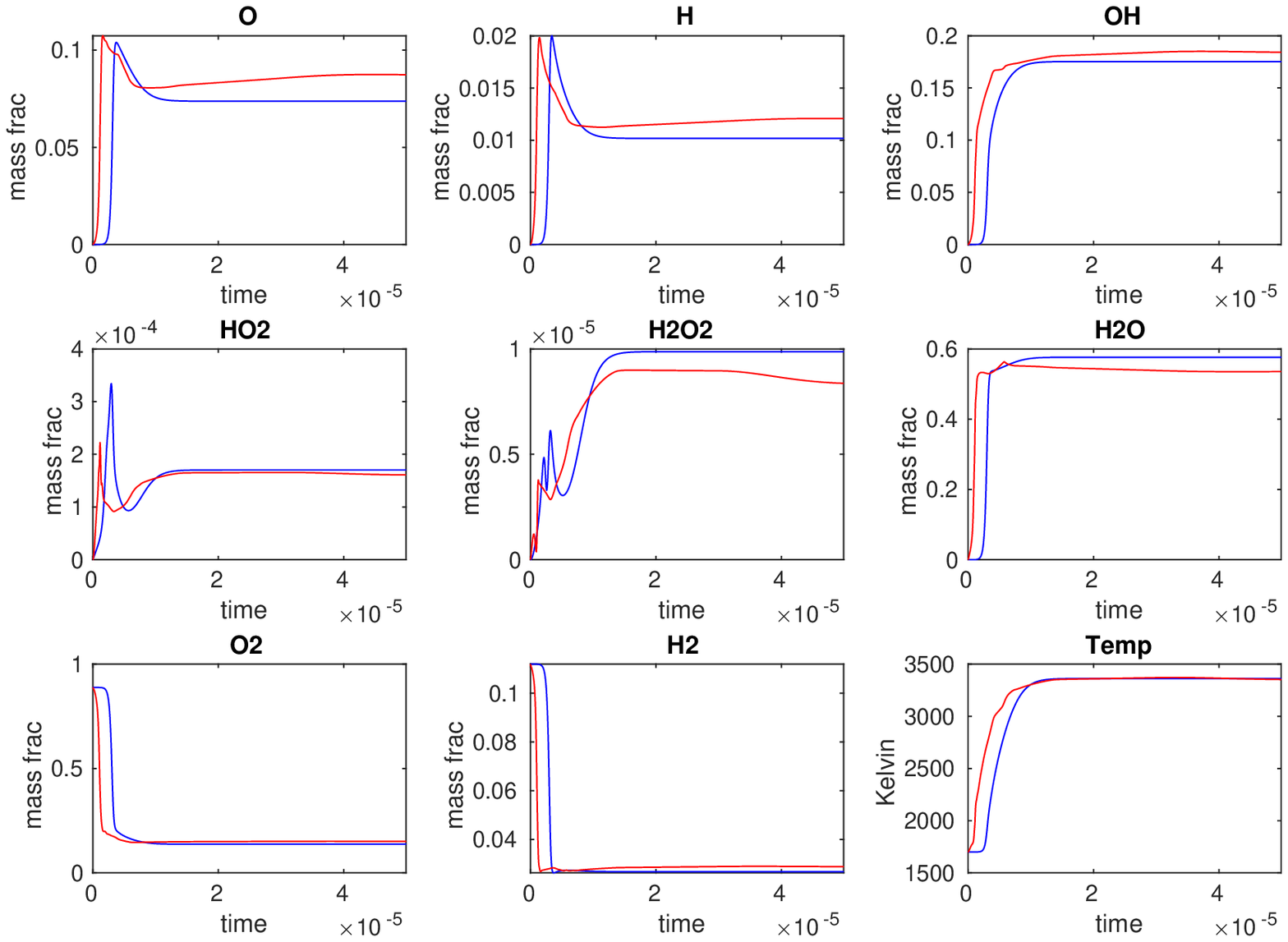}
    \includegraphics[width=0.495\textwidth]{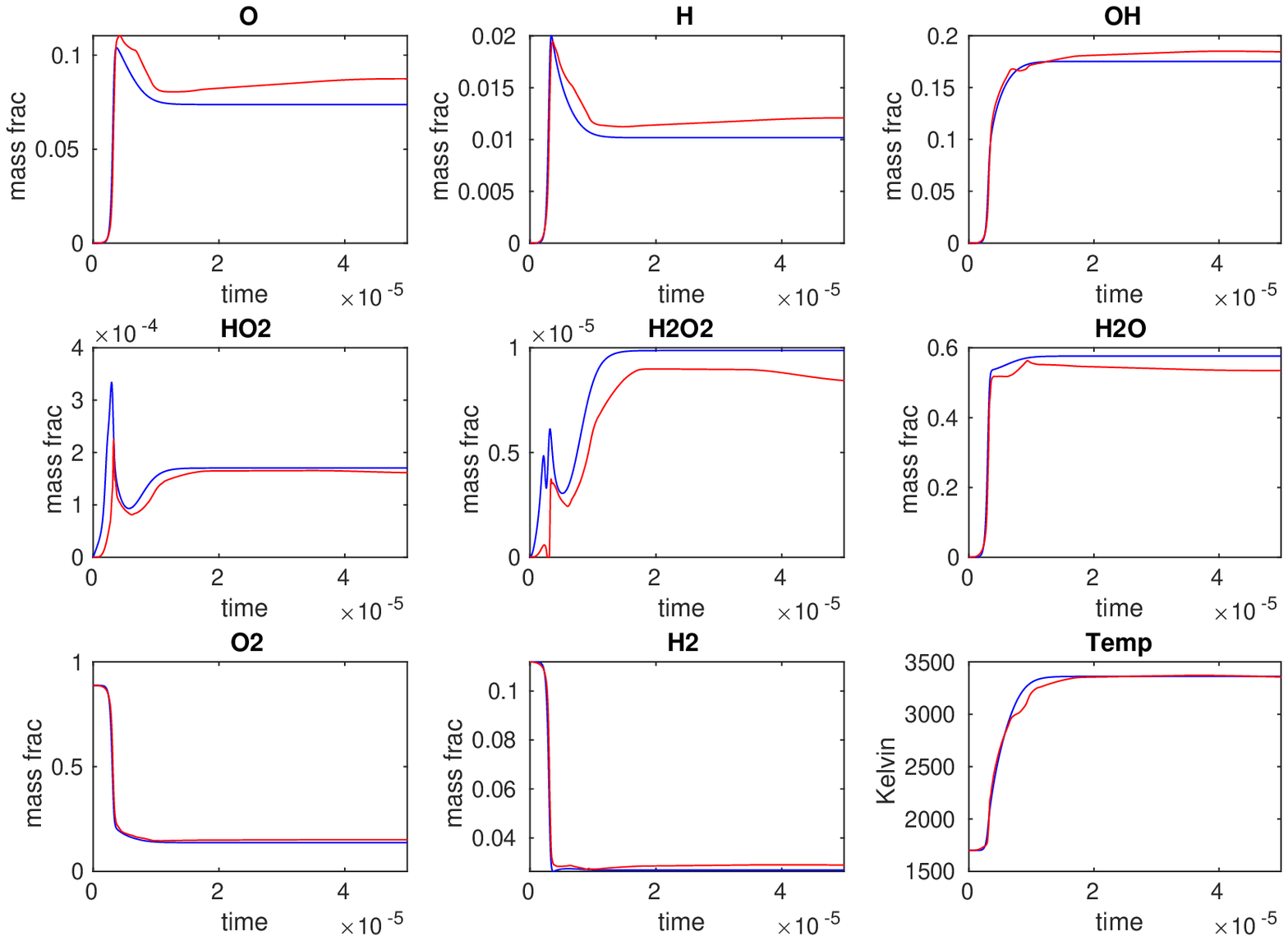}
    \caption{Left: No nudging. Right: Nudging towards initial data. Initial temperature 1700K. Blue is the CHEMKIN solution and red is the ResNet solution.}
    \label{fig:sidebyside4}
\end{figure}

\section{Conclusions}
This paper has introduced nudging induced neural networks (NINNs). NINNs can be implemented onto pre-existing neural networks which allow the user to control the neural network. We demonstrated how NINNs are effective as data assimilation algorithms on the Lorenz 63 and Lorenz 96 ODEs. NINNs were able to outperform the classical nudging data assimilation algorithm in our experiments while retaining the ease of computation found in neural networks. We demonstrated the uses of NINNs in replacing stiff ODE dynamics with a neural network. The convergence analysis gave us insight into the errors involved with NINNs and how to minimize them. 

\bibliographystyle{siamplain}
\bibliography{references}
\end{document}